\documentclass{amsart}

\usepackage{graphicx}
\usepackage{amsmath}
\usepackage{multicol}
\usepackage{mathrsfs}
\usepackage{amssymb}
\usepackage{xcolor}
\usepackage{enumerate}
\usepackage{bm}
\usepackage{float}
\usepackage{subfig}
\usepackage{fancyvrb}
\usepackage{booktabs}
\usepackage{hyperref}

\usepackage[onehalfspacing]{setspace}
\usepackage[top=100pt,bottom=100pt,left=100pt,right=110pt,footskip=30pt]{geometry}

\allowdisplaybreaks

\theoremstyle{plain}
\newtheorem{theorem}{Theorem}[section]
\newtheorem{corollary}[theorem]{Corollary}
\newtheorem{lemma}[theorem]{Lemma}

\newtheorem*{theorem*}{Theorem}
\newtheorem*{corollary*}{Corollary}
\newtheorem*{lemma*}{Lemma}
\newtheorem*{proposition*}{Proposition}

\theoremstyle{definition}
\newtheorem{definition}[theorem]{Definition}

\newtheorem*{definition*}{Definition}
\newtheorem*{example*}{Example}
\newtheorem*{problem*}{Problem}
\newtheorem*{exercise*}{Exercise}
\newtheorem{remark}{Remark}

\newcommand{\bbE}{\mathbf{E}}

\newcommand{\bbN}{\mathbb{N}}

\newcommand{\bbP}{\mathbf{P}}

\newcommand{\bbR}{\mathbb{R}}

\newcommand{\bfB}{\mathbf{B}}

\newcommand{\bfE}{\mathbf{E}}

\newcommand{\bfW}{\mathbf{W}}

\newcommand{\ind}{\mathbf{1}}

\newcommand{\bfs}{\mathbf{s}}
\newcommand{\bft}{\mathbf{t}}

\newcommand{\bfx}{\mathbf{x}}

\newcommand{\cL}{\mathcal{L}}

\newcommand{\cN}{\mathcal{N}}

\newcommand{\norm}[1]{\left\lVert#1\right\rVert}

\newcommand{\R}{\mathbb{R}}
\newcommand{\N}{\mathbb{N}}

\def\sas{S$\aa$S}

\def\bi{\begin{itemize}}
	\def\ei{\end{itemize}}
\def\ll{\ell_{\text{lev}}}
\def\aa{\alpha}

\def\eps{\epsilon}

\def\be{\begin{equation}}
	\def\ee{\end{equation}}
\def\bea{\begin{eqnarray}}
	\def\eea{\end{eqnarray}}
\def\nn{\nonumber}
\def\ff{\infty}
\def\({\left(}
\def\){\right)}

\def\weak{\stackrel{\text{w}}{\to}}

\def\1{\mathbf{1}}

\def \bx{\mathbf{x}}
\def \bn{\mathbf{n}}
\def \E{\mathbf{E}}
\def \P{\mathbf{P}}

\def \bW{\mathbf{W}}
\def \bX{\mathbf{X}}
\def \bY{\mathbf{Y}}
\def \by{\mathbf{y}}

\author{Paul Jung}
\address[PJ]{KAIST Department of Mathematical Sciences, Daejeon, Korea.}
\email{\url{mailto:pauljung(at)kaist.ac.kr}}
\urladdr{\url{http://mathsci.kaist.ac.kr/~pauljung/}}

\author{Hoil Lee}
\address[HL]{KAIST Department of Mathematical Sciences, Daejeon, Korea.}
\email{\url{mailto:hoil.lee(at)kaist.ac.kr}}

\author{Jiho Lee}
\address[JL]{KAIST Department of Mathematical Sciences, Daejeon, Korea.}
\email{\url{mailto:efidiaf(at)gmail.com}}

\author{Hongseok Yang}
\address[HY]{KAIST School of Computing, Daejeon, Korea.}
\email{\url{mailto:hongseok00(at)gmail.com}}
\urladdr{\url{https://sites.google.com/view/hongseokyang/home}}

\title{{\large $\alpha$-Stable convergence of heavy-tailed infinitely-wide neural networks}}

\begin{document}
\pagestyle{plain}

\begin{abstract} 
        We consider infinitely-wide multi-layer perceptrons (MLPs) which are limits of standard deep feed-forward neural networks. We assume that, for each layer, the weights of an MLP are initialized with i.i.d. samples from either a light-tailed (finite variance) or heavy-tailed distribution in the domain of attraction of a symmetric $\alpha$-stable distribution, where $\alpha\in(0,2]$ may depend on the layer. For the bias terms of the layer, we assume i.i.d. initializations with a symmetric $\alpha$-stable distribution having the same $\alpha$ parameter of that layer. We then extend a recent result of Favaro, Fortini, and Peluchetti (2020), to show that the vector of pre-activation values at all nodes of a given hidden layer converges in the limit, under a  suitable scaling, to a vector of i.i.d. random variables with symmetric $\alpha$-stable distributions.
\end{abstract}

\maketitle

\section{Introduction}

Deep neural networks have brought remarkable progresses in a wide range of applications, such as language translation and speech recognition, but a satisfactory mathematical answer on why they are so effective has yet to come. One promising direction, with a large amount of recent research activity, is to analyze neural networks in an idealized setting where the networks have infinite widths and the so-called step size becomes infinitesimal. In this idealized setting, seemingly intractable questions can be answered. For instance, it has been shown that as the widths of deep neural networks tend to infinity, the networks converge to Gaussian processes, both before and after training, if their weights are initialized with i.i.d. samples from the  Gaussian distribution~\cite{neal1996book,lee2018nngp,matthews2018nngp,novak2019nngp-cnn,yang2019nngp-general}. (The methods used in these works can easily be adapted to show convergence to Gaussian processes when the initial weights are i.i.d. with finite variance.) Furthermore, in this setting, the training of a deep neural network (under the standard mean-squared loss) is shown to achieve zero training error, and the analytic form of a fully-trained network with zero error has been identified~\cite{jacot2018ntk,lee2019ntk}. These results, in turn, enable the use of tools from stochastic processes and differential equations for analyzing deep neural networks in a novel way. They have also led to new high-performing data-analysis algorithms based on Gaussian processes~\cite{lee2020nngp-ntk-experiment}.

We extend this line of research on infinitely-wide deep neural networks by going beyond finite-variance distributions as initializers of network weights. We consider deep networks whose weights in a given layer are allowed to be initialized with i.i.d. samples {from {\it either} a light-tailed (finite variance) or heavy-tailed distribution in the domain of attraction of a symmetric stable distribution, and show that as the widths of the networks increase, the networks at initialization converge to symmetric $\alpha$-stable processes.} Although neural networks with possibly heavy-tailed initializations are not common, their potential for modeling heavy-tailed data was recognized early on by~\cite{wainwright1999scalemixture}, and even the convergence of an infinitely-wide yet shallow neural network under non-Gaussian initialization was shown in the 90's~\cite{neal1996book}. Recently, Favaro, Fortini, and Peluchetti extended such convergence results from shallow to deep networks~\cite{favaro2020stable}. Our work is built on this result, as we next explain.

Favaro et al. considered multi-layer perceptrons (MLPs) having large width $n$, and having i.i.d. weights with a symmetric $\alpha$-stable (\sas) distribution of scale parameter $\sigma_w$. A random variable $X$ is said to have a \sas\ distribution, if its characteristic function takes the form, for $0<\aa\le 2$,
$$\psi_X(t):= \bbE e^{itX} = e^{-|\sigma t|^\alpha}$$
for some constant $\sigma>0$ called the scale parameter. In the special case $\aa=2$, $X$ has a Gaussian distribution with variance $2\sigma^2$ (which differs from standard notation in this case, by a factor of 2).

The results of  Favaro et al. showed that as $n$ tends to $\ff$, the arguments of the nonlinear activation function $\phi$, in any given hidden layer, converge jointly in distribution to a product of \sas($\sigma_\ell$)\ distributions with the same $\alpha$ parameter. The scale parameter $\sigma_\ell$ differs for each layer $\ell$; however, an explicit form was provided as a function of $\sigma_w$, the input $\bx = (x_1,\ldots,x_I)$, and the distribution of bias terms which have a \sas($\sigma_{B}$) distribution for some $\sigma_{B}>0$. They also showed that as a function of $\bx$, the joint distribution described above is an $\alpha$-stable process and described the spectral measure (\cite[Sec. 2.3]{samorodnitsky1994stable}) of this process at points $\bx_1,\ldots,\bx_n$.

Here, we aim to show universality in the sense that the results hold also when the weights are i.i.d. and heavy-tailed, and in the domain of attraction of a \sas\ distribution. Also, part of our goal is to clarify some details of the proof in \cite{favaro2020stable} and fill in the details of one nontrivial step in the proof of \cite{favaro2020stable} (for instance our use of Lemma \ref{lem:mixing measure}). Furthermore, we will consider a slightly more general case where the $\alpha$ parameter for the weights may depend on the layer it is in, {including the case where it may be that $\alpha=2$ for some layers.}

\vspace{2mm}
\paragraph{\bf Notation} Let $\Pr(\bbR)$ be the set of probability distributions on $\R$. In the sequel, for $\alpha\in(0,2]$, let $\mu_{\alpha,\sigma}\in \Pr(\bbR)$ denote a \sas($\sigma$) distribution. We will typically use capital letters to denote random variables in $\R$.
For example, the random weights of our MLPs in layer $\ell$ are denoted $(W_{ij}^{(\ell)})_{ij}$ which are henceforth assumed to be in the domain of attraction of $\mu_{\alpha,\sigma}$, which may depend on $\ell$. One notable exception to this convention is our use of the capital letter $L$ to denote a slowly varying function.

\section{The Model: Heavy-tailed Multi-layer Perceptrons}

At a high level, a neural network is just a parameterized function $Y$ from the inputs in $\R^I$ to outputs in $\R^O$ for some $I$ and $O$. In this article, we consider the case that $O = 1$. The parameters $\Theta$, of the function, consist of real-valued vectors $\bfW$ and $\bfB$, called weights and biases. These parameters are initialized randomly, and get updated repeatedly during the training of the network. 
We adopt the common notation $Y_\Theta(\bx)$, and express that the output of $Y$ depends on both the input $\bx$ and the parameters  {$\Theta=(\bfW,\bfB)$}. 

Note that since $\Theta$ is set randomly, $Y_\Theta$ is a random function. This random-function viewpoint is the basis of a large body of work on Bayesian neural networks~\cite{neal1996book}, which studies the distribution of this random function or its posterior conditioned on input-output pairs in training data. Our work falls into this body of work. We analyze the distribution of the random function $Y_\Theta$ at the moment of initialization.
Our analysis is in the situation where $Y_\Theta$ is defined by an MLP, the width of the MLP is large (so the number of parameters in $\Theta$ is large), and the parameters $\Theta$ are initialized by possibly using heavy-tailed distributions. The precise description of the setup is given below.
\begin{enumerate}
\item (Weights and Biases)
The MLP is fully connected, and the weights on the edges from layer $\ell-1$ to $\ell$ are given by $ \bW^{(\ell)} = (W^{(\ell)}_{ij})_{ij \in \N^2}$.
Assume that $\bW^{(\ell)}$ is a collection of i.i.d. symmetric random variables such that for each layer $\ell$,
\begin{itemize}
	\item[(2.1.a)] they are heavy-tailed, i.e. for all $t>0$,
	\begin{align} \label{eq:ht}
	\bbP(|W^{(\ell)}_{ij}| > t) = t^{-\alpha_\ell}L^{(\ell)}(t),\qquad  {\text{for some }\alpha_\ell\in(0,2]},
	\end{align}
	where $L^{(\ell)}$ is some slowly varying function, or
	\item[(2.1.b)] $\E |W^{(\ell)}_{ij}|^{2}<\infty$. (In this case, we set $ \alpha_{\ell} = 2 $ by default.)
\end{itemize}
Note that both (2.1.a) and (2.1.b) can hold at the same time.
Even when this happens, there is no ambiguity about $\alpha_{\ell}$, which is set to be $2$ in both cases. Our proof deals with the cases when $ \alpha_{\ell} <2 $ and $\alpha_{\ell} =2 $ separately. (See below, the definition of $ L_{0}$.)
We permit both the conditions (2.1.a) and (2.1.b) to emphasize that our result covers a mixture of both heavy-tailed and finite variance (light-tailed) initializations.

Let $B^{(\ell)}_{i}$ be i.i.d. random variables with distribution $\mu_{\alpha_\ell,\sigma_{B^{(\ell)}}}.$ 
Note that the distribution of $B^{(\ell)}_i$ is more constrained than that of $W^{(\ell)}_{ij}$. This is because the biases are not part of the normalized sum, and normalization is, of course, a crucial part of the stable limit theorem.

For later use in the $ \alpha=2 $ case, we define a function $ \widetilde{L}^{(\ell)} $ by
{\begin{align*}
	\widetilde{L}^{(\ell)}(x)
        := \int_{0}^{x} y \P(|W_{ij}^{(\ell)}|>y) \, dy.
\end{align*}
Note that $ \widetilde{L}^{(\ell)} $ is increasing. For case (2.1.b), $\int_{0}^{x} y \P(|W_{ij}^{(\ell)}|>y) \, dy$ converges to a constant, namely to  $1/2$ of the variance, and thus it is slowly varying. For case (2.1.a), it is seen in Lemma \ref{lem: tilde L} that $ \widetilde{L}^{(\ell)} $ is slowly varying as well.}

For convenience, let
\begin{align*}
	L_{0} :=
	\begin{cases}
		L^{(\ell)} \quad &\text{if } \aa_{\ell} < 2 \\
		\widetilde{L}^{(\ell)} &\text{if } \aa_{\ell} = 2.
	\end{cases}
\end{align*}
We have dropped the superscript $ \ell $ from $L_0$ as the dependence on $ \ell $ will be assumed.

\item (Layers)
We suppose that there are $\ll$ layers, not including those for the input and output. The $0$-th layer is for the input and consists of $I$ nodes assigned with deterministic values from the input $\bx=(x_1,\ldots,x_I)$. We assume for simplicity that $x_i\in\R$.\footnote{None of our methods would change if we instead let $x_i\in\R^d$ for arbitrary finite $d$.} The layer $\ll+1$ is for the output.
\item
(Scaling) Fix a layer $\ell$ with $2\le \ell\le\ll+1$, and let $n$ be the number of nodes at the layer $\ell-1$. We will scale the random values at the nodes (pre-activation) by
\begin{align*}
	a_{n}(\ell):= \inf\{ t > 0 \colon t^{-\alpha_\ell}L_{0}(t) \le n^{-1} \}.
\end{align*}
Then, $ a_{n}(\ell)$ tends to $\infty$ as $n$ increases.
For future purposes we record the well-known fact that, for $a_n=a_n(\ell)$,
                \footnote{For case (2.1.b), $t^2 L_0(t)$ becomes continuous and so $n a_n^{-\alpha_{\ell}}L_0(a_n)$ is simply $1$.  To see the convergence in case (2.1.a), first note that as $ \bbP(|W^{(\ell)}_{ij}|>t) = t^{-\alpha_{\ell}}L^{(\ell)}(t) $ is right-continuous, $ n a_{n}^{-\alpha_{\ell}} L^{(\ell)}(a_{n}) \le 1 $.
	For the reverse inequality, note that by \eqref{eq:ht} and the definition of $a_n$, for $n$ large enough we have $\bbP\left(|W^{(\ell)}_{ij}| > \frac{1}{1+\epsilon}a_{n}\right) \geq 1/n$, and by the definition of slowly varying that,
	\begin{align*}
        (1+2\epsilon)^{-\alpha_{\ell}}
	= \lim_{n\to \infty} \frac{\bbP\left(|W^{(\ell)}_{ij}| > \frac{1+2\epsilon}{1+\epsilon}a_{n}\right)}{\bbP\left(|W^{(\ell)}_{ij}| > \frac{1}{1+\epsilon}a_{n}\right)}
	\le \liminf_{n\to\infty} \frac{\bbP\left(|W^{(\ell)}_{ij}| > a_{n}\right)}{1/n}  .
	\end{align*}.
}
\begin{align}\label{a_n asymptotics}
\lim_{n\to\ff}n a_{n}^{-\alpha_\ell} L_0(a_{n}) = 1 .
\end{align}

\item (Activation)
The MLP uses a nonlinear activation function $\phi(y)$. We assume that $ \phi $ is continuous and bounded. The boundedness assumption simplifies our presentation, and in Section~\ref{sec: Relaxing the Boundedness Assumption}, we relax this assumption so that for particular initializations (such as Gaussian or stable), more general activation functions such as ReLU are allowed.

\item (Hidden Layers) For layer $\ell$ with $1\le\ell\le\ll$, there are $n_\ell$ nodes for some $n_\ell \ge 2$. We write $\bn=(n_1,\ldots, n_{\ll})\in\bbN^{\ll}$. For $ \ell $ with $ 1\le\ell\le\ll+1 $, the pre-activation values at these nodes are given, for an input $ \bx \in \bbR^{I} $, recursively by
\begin{align*}
&Y^{(1)}_{i}(\bx;\bn) := Y^{(1)}_{i}(\bx) := \sum_{j=1}^{I}W^{(1)}_{ij} x_{j} + B^{(1)}_{i}, \\
&Y^{(\ell)}_{i}(\bx;\bn) := \frac{1}{a_{n_{\ell-1}}(\ell)} \sum_{j=1}^{n_{\ell-1}} W^{(\ell)}_{ij} \phi(Y^{(\ell-1)}_{j}(\bx;\bn)) + B^{(\ell)}_{i} , \quad \ell \ge 2
\end{align*}
for each $ n_{\ell-1} \in \bbN $ and $ i \in \bbN $. We often omit $\bn$ and write $Y^{(\ell)}_{i}(\bx)$. When computing the output of the MLP with widths $\bn$, one only needs to consider $i\le n_{\ell}$ for each layer $\ell$. However, it is always possible to assign values to an extended MLP beyond $\bn$ which is why we have assumed more generally that $ i \in \bbN $. This will be important for the proofs as explained in Remark \ref{rem:1} below.

{Note that $Y^{(\ell)}_{i}(\bx;\bn)$ depends on only the coordinates $n_1,\ldots,n_{\ell-1}$, but we may simply let it be constant in the coordinates $n_\ell,\ldots,n_{\ll}$. This will often be the case when we have functions of $\bn$ in the sequel.}
\item (Limits) We consider one MLP for each $\bn\in\bbN^{\ll}$. We take the limit of the collection of these MLPs in such a way that 
\begin{align}\label{n to infinity}
\min(n_1,\ldots,n_{\ll})\to\infty.
\end{align}
 (Our methods can also handle the case where limits are taken from left to right, i.e., $\lim_{n_{\ll}\to\infty}\cdots\lim_{n_1\to\infty}$, but since this order of limits is easier to prove, we will focus on the former.)

\end{enumerate}

\section{Convergence to $\alpha$-stable Distributions}
		
		Our main results are summarized in the next theorem and its extension to the situation of multiple inputs in Theorem~\ref{theorem:multivariate-main} in Section~\ref{section:joint}. They show that as the width of an MLP tends to infinity, the MLP becomes a relatively simple random object: the outputs of its $\ell$-th layer become just i.i.d. random variables drawn from a stable distribution, and the parameters of the distribution have explicit inductive characterizations.
		
				Let $$c_\alpha := \lim_{M \to \infty} \int_{0}^{M} \frac{\sin u}{u^{\alpha}} \, du \quad \text{for } \aa < 2 \quad \text{and} \quad c_{2}=1.$$
		
		\begin{theorem}\label{prop:main}
			For each $ \ell = 2,\ldots,\ll+1 $, the joint distribution of $ (Y^{(\ell)}_{i}(\bx;\bn))_{i \ge 1}$ converges weakly to $\bigotimes_{i \ge 1} \mu_{\alpha_\ell,\sigma_\ell} $ as $\min(n_1,\ldots,n_{\ll})\to\infty$, with $ \sigma_{\ell} $ inductively defined by
			\begin{align*}
			\sigma_{2}^{\alpha_2}&:= \sigma_{B^{(2)}}^{\alpha_2} + c_{\alpha_2} \int |\phi(y)|^{\alpha_2} \, \nu^{(1)}(dy), \quad \ell = 2, \\
				\sigma_{\ell}^{\alpha_\ell} &:= \sigma_{B^{(\ell)}}^{\alpha_\ell} + c_{\alpha_\ell} \int |\phi(y)|^{\alpha_\ell} \, \mu_{\alpha_{\ell-1},\sigma_{\ell-1}}(dy), \quad \ell = 3,\ldots,\ll+1
			\end{align*}
			where $\nu^{(1)}$ is the distribution of  $Y^{(1)}_{1}(\bx)$.
			That is, the characteristic function of the limiting distribution 
                        is, for any finite subset $ \cL \subset \bbN $,
			\begin{align*}
			&\prod_{i\in\cL} \psi_{B^{(2)}}(t_i)\exp\left( -c_{\alpha_2}  |t_{i}|^{\alpha_2} \int |\phi(y)|^{\alpha_2} \, \nu^{(1)}(dy) \right) , \quad \ell = 2, \\
			&\prod_{i\in\cL} \psi_{B^{(\ell)}}(t_i)\exp\left( -c_{\alpha_\ell}  |t_{i}|^{\alpha_\ell} \int |\phi(y)|^{\alpha_\ell} \, \mu_{\alpha_{\ell-1},\sigma_{\ell-1}}(dy) \right) , \quad \ell = 3,\ldots,\ll+1 .
			\end{align*}
		\end{theorem}
	
	{\begin{remark}\label{rmk: boundedness of phi}
		The integrals in Theorem \ref{prop:main} are well-defined since $ \phi $ is bounded. For (possibly) unbounded $ \phi $, these integrals are well-defined as well under suitable assumption on $ \phi $. See Section~\ref{sec: Relaxing the Boundedness Assumption}.
	\end{remark}}
	
	\begin{remark}\label{rem:1}
		Before embarking on the proof, let us make an important remark.
	 	For each $\bn=(n_1,\ldots,n_{\ll})$, the MLP is finite and each layer has finite width. A key part of the proof is the application of de Finetti's theorem at each layer, which applies only in the case where one has an infinite sequence of random variables (for a given layer, our sequence is such that there is one random variable at each node). 
		As in \cite{favaro2020stable}, a crucial observation is that for each $\bn=(n_1,\ldots,n_{\ll})$, we can extend the MLP to an infinite-width MLP by adding an infinite number of nodes at each layer that compute values in the same manner as nodes of the original MLP, but are ignored by nodes at the next layer. Thus, the finite-width MLP is embedded in an infinite-width MLP. This allows us to use de Finetti's theorem.
		\end{remark}

\vspace{2mm}
\paragraph{\bf Heuristic of the proof}
The main takeaway of the theorem is that, even though the random variables {$(Y_i^{(\ell)}(\bx;\bn))_{i \in \N}$ are dependent through the randomness of the former layer's outputs $(Y_j^{(\ell-1)}(\bx;\bn))_{j \in \N}$,} as the width grows to infinity, this dependence vanishes via an averaging effect.
Let us briefly highlight the key technical points involved in establishing this vanishing dependence on a heuristic level. 

By de Finetti's theorem, {for each $\bn$} there exists a random distribution  {$\xi^{(\ell-1)}(dy {;\bn})$} such that the sequence $(Y_j^{(\ell-1)}(\bx))_j$ is conditionally i.i.d. with common random distribution $\xi^{(\ell-1)}.$ By conditioning on $\xi^{(\ell-1)}$, we obtain independence among the summands of $$Y^{(\ell)}_i(\bx) = \frac{1}{a_{n_{\ell-1}}(\ell)} \sum_{j=1}^{n_{\ell-1}} W^{(\ell)}_{ij} \phi(Y^{(\ell-1)}_{j}(\bx)) + B^{(\ell)}_{i}$$
as well as independence among the family $(Y_i^{(\ell)}(\bx))_i$. Let
$\alpha := \alpha_\ell$, $n := n_{\ell - 1}$, and $a_n := a_{n_{\ell - 1}}(\ell)$.
With the help of Lemma \ref{lem: pitman approximation}, the conditional characteristic function of $Y_1^{(\ell)}(\bx)$ given $\xi^{(\ell-1)}$ is asymptotically equal to
\begin{align}\label{eq:heuristic}
  e^{-\sigma_{B}^{\alpha} |t|^{\alpha}} \left( 1 - \frac{b_n }{n} c_\alpha |t|^\alpha \int |\phi(y)|^{\alpha}\frac{L_{0}\left( \frac{a_{n}}{|\phi(y)t|} \right)}{L_{0}(a_{n})} \, \xi^{(\ell-1)}(dy {;\bn}) \right)^n,
\end{align}
where $b_n$ is a deterministic constant that tends to one. Assuming the inductive hypothesis, the random distribution $\xi^{(\ell-1)}$ converges weakly to $\mu_{\alpha_{\ell-1}, \sigma_{\ell-1}}$   {as $\bn\to\infty$} in the sense of \eqref{n to infinity}, by Lemma \ref{lem:mixing measure}. Since $L_0$ is slowly varying, one can surmise that the conditional characteristic function tends to
$$ \exp \left( -\sigma_{B}^{\alpha} |t|^{\alpha} -c_\alpha |t|^\alpha \int |\phi(y)|^\alpha \mu_{\alpha_{\ell-1}, \sigma_{\ell-1}}(dy)\right),$$
which is the characteristic function of the stable law we desire. To make the above intuition rigorous, the convergence of \eqref{eq:heuristic} is verified by proving uniform integrability of the integrand $|\phi(y)|^{\alpha}\frac{L_{0}\left( \frac{a_{n}}{|\phi(y)t|} \right)}{L_{0}(a_{n})}$ with respect to the family of distributions $\xi^{(\ell-1)}$ over the indices $\bn$. Namely, by Lemma \ref{lem: pitman slowly varying bound}, the integrand can be bounded by $O(|\phi(y)|^{\alpha \pm \epsilon})$ for small $\epsilon>0$ and uniform integrability follows from the boundedness of $ \phi $. The joint limiting distribution converges to the desired stable law by similar arguments.

			\begin{proof}[Proof of Theorem \ref{prop:main}]
				
				We start with a useful expression for the characteristic function conditioned on the random variables $ \{ Y^{(\ell-1)}_{j}(\bx) \}_{j=1,\ldots,n_{\ell-1}} $:
                                \begin{align}
                                        \label{eq: conditional ch f}
                                                & \psi_{Y^{(\ell)}_{i}(\bx) | \{ Y^{(\ell-1)}_{j}(\bx) \}_{j}}(t)
                                                := \bbE\left[ \left. \exp\left ( it Y^{(\ell)}_i(\bx)\right ) \right| \{ Y^{(\ell-1)}_{j}(\bx)\}_j \right] 
                                                \\
                                                \nonumber
                                                &\qquad\qquad {} = \bbE\left[ \left. \exp\left ( it \left \{ \frac{1}{a_{n_{\ell-1}}(\ell)} \sum_{j=1}^{n_{\ell-1}} W^{(\ell)}_{ij} \phi(Y^{(\ell-1)}_{j}(\bx)) + B^{(\ell)}_{i} \right \} \right ) \right| \{ Y^{(\ell-1)}_{j}(\bx)\}_j \right] 
                                                \\
                                                \nonumber
                                                &\qquad\qquad {} = e^{-\sigma |t|^{\alpha_\ell}} \prod_{j=1}^{n_{\ell-1}} \psi_{W^{(\ell)}_{ij}}\left (\frac{\phi(Y^{(\ell-1)}_{j}(\bx))}{a_{n_{\ell-1}}(\ell)}t\right )
                                \end{align}
								where $\sigma:=\sigma_{B^{(\ell)}}^{\alpha_\ell}$ and the argument on the right-hand side is random.
				
				\noindent{\bf Case $\ell=2$:}
				
				Let us first consider the case $ \ell=2 $. Let $n=n_1, \alpha=\alpha_2$, $a_n=a_{n_1}(2)$, and $t\neq 0$. 
				We first show the weak convergence of the one-point marginal distributions, i.e., we show that the distribution of $ Y^{(2)}_{i}(\bx)$ converges weakly to $\mu_{\alpha,\sigma}$ for each $ i $.
				Since $ Y^{(1)}_{j}(\bx), j=1,\ldots,n $ are i.i.d., this is a straight-forward application of standard arguments, which we include for completeness.
				Denote the common distribution of  $ Y^{(1)}_{j}(\bx), j=1,\ldots,n $  by $\nu^{(1)}$.
				Taking the expectation of \eqref{eq: conditional ch f} with respect to the randomness of $ \{Y^{(1)}_{j}(\bx)\}_{j=1,\ldots,n} $,
				\begin{align*}
					\psi_{Y^{(2)}_{i}(\bx)}(t)
					&= e^{-\sigma_{B^{(2)}}^{\alpha} |t|^{\alpha}} \left( \int \psi_W\left (\frac{\phi(y)}{a_{n}}t\right ) \, \nu^{(1)}(dy) \right)^{n}
				\end{align*}
					where $\psi_{W}:=\psi_{W_{ij}^{(2)}}$ for some/any $i,j$.
				From Lemma \ref{lem: pitman approximation}, we have that
				\begin{align*}
					\psi_W(t)
					= 1 - c_\alpha|t|^{\alpha} L_{0}\left( \frac{1}{|t|} \right) + o\left(|t|^{\alpha} L_{0}\left( \frac{1}{|t|} \right)\right), \quad |t| \to 0
				\end{align*}
				for $ c_{\alpha} = \lim_{M\to \infty}\int_{0}^{M} \sin u / u^{\alpha} \, du $ when $ \alpha < 2 $ and $ c_{2} =1 $.
				If $\phi(y)=0$ then $\psi_W\left (\frac{\phi(y)}{a_{n}}t\right )=1$.
				Otherwise, setting $ b_n  := n a_{n}^{-\alpha}L_{0}(a_{n})$, for fixed $ y $ with $ \phi(y) \ne 0 $  we have that, as $ n \to \infty $,
				\begin{align}\label{eq: phi expansion}
					\psi_W\left (\frac{\phi(y)}{a_{n}}t\right )
					= 1 - c_\alpha \frac{b_n }{n} |\phi(y) t|^{\alpha} \frac{L_{0}\left( \frac{a_{n}}{|\phi(y)t|} \right)}{L_{0}(a_{n})} + o\left( \frac{b_n }{n} |\phi(y) t|^{\alpha} \frac{L_{0}\left( \frac{a_{n}}{|\phi(y)t|} \right)}{L_{0}(a_{n})} \right).
				\end{align}
				By Lemma \ref{lem: pitman slowly varying bound} applied to $ G(x) := x^{-\alpha}L_{0}(x) $ and $ c=1 $, 
                                for any $\epsilon > 0$, there exist constants $ b > 0$ and $ n_0 $ such that for all $ n > n_0 $ and all $y$ with $\phi(y) \neq 0$,
				\begin{align}
                                        \label{eqn:prop-main:base-case:upper-bound}
					|\phi(y) t|^{\alpha} \frac{L_{0}\left( \frac{a_{n}}{|\phi(y)t|} \right)}{L_{0}(a_{n})} 
                                        = \frac{G\left(\frac{a_{n}}{|\phi(y)t|}\right)}{G(a_{n})}
					\leq b |\phi(y)t|^{\alpha \pm \epsilon},
				\end{align}
				where $|\cdot|^{\alpha\pm\epsilon}$ denotes the maximum of $|\cdot|^{\alpha +\epsilon}$ and $|\cdot|^{\alpha -\epsilon}$.

				Since $ \phi $ is bounded, the right-hand side of (\ref{eq: phi expansion}) is term-by-term integrable with respect to $ \nu^{(1)}(dy) $. In particular, the integral of the error term can be bounded, for some small $ \epsilon $ and large enough $n$, by 
				\begin{align*}
					\int o\(\frac{b_n }{n} |\phi(y) t|^{\alpha} \frac{L_{0}\left( \frac{a_{n}}{|\phi(y)t|} \right)}{L_{0}(a_{n})}\)\, \nu^{(1)}(dy)
					\le o\(b \frac{b_n }{n} \int |\phi(y)t|^{\alpha \pm \epsilon} \, \nu^{(1)}(dy)\)
					= o\(\frac{b_n }{n} \) .
				\end{align*}
                              (Set $|\phi(y)|^{\alpha}L_{0}(\frac{a_{n}}{|\phi(y)|})=0$ when $\phi(y)=0$.)
				Thus, integrating both sides of (\ref{eq: phi expansion}) with respect to $ \nu^{(1)}(dy) $ and taking the $ n $-th power, it follows that
				\begin{align*}
					\left(\int \psi_W\left (\frac{\phi(y)t}{a_{n}}\right ) \, \nu^{(1)}(dy)\right)^{n}
					= \left(1 - c_\alpha\frac{b_n }{n} \int |\phi(y)t|^{\alpha}\frac{L_{0}\left( \frac{a_{n}}{|\phi(y)t|} \right)}{L_{0}(a_{n})} \, \nu^{(1)}(dy) + o\left ( \frac{b_n }{n} \right )\right)^{n} .
				\end{align*}
                From the bound in \eqref{eqn:prop-main:base-case:upper-bound}, we have, by dominated convergence, that as $n\to\infty$
				\begin{align*}
					\int |\phi(y)t|^{\alpha}\frac{L_{0}\left( \frac{a_{n}}{|\phi(y)t|} \right)}{L_{0}(a_{n})} \, \nu^{(1)}(dy)
					\to |t|^{\alpha} \int |\phi(y)|^{\alpha} \, \nu^{(1)}(dy).
				\end{align*}
				Since $  b_n  = n a_{n}^{-\alpha}L_{0}(a_{n})$ converges to $1$ by \eqref{a_n asymptotics}, we have that
				\begin{align*}
					\left(\int \psi_W\left (\frac{\phi(y)t}{a_{n}}\right ) \, \nu^{(1)}(dy)\right)^{n}
					\to \exp\left( -c_\alpha|t|^{\alpha} \int |\phi(y)|^{\alpha} \, \nu^{(1)}(dy) \right) .
				\end{align*}
				Thus, the distribution of $ Y^{(2)}_{i}(\bx) $ weakly converges to $\mu_{\alpha,\sigma_2}$ where
				\begin{align*}
					\sigma_2^{\alpha} &= \sigma_{B^{(2)}}^{\alpha} + c_\alpha \int |\phi(y)|^{\alpha} \, \nu^{(1)}(dy)
				\end{align*} as desired.
				
				Next we prove that the joint distribution of $ (Y^{(2)}_{i}(\bx))_{i \ge 1}$ converges to the product distribution $\bigotimes_{i \ge 1} \mu_{\alpha,\sigma_2}$.
				Let $ \cL \subset \bbN $ be a finite set.
				{Let $\psi_B$} denote the multivariate characteristic function for the $|\cL|$-fold product distribution of $\mu_{\alpha,\sigma_{B^{(2)}}}.$
				For $ \bft = (t_{i})_{i\in\cL} $, conditionally on $ \{ Y^{(1)}_{j}(\bx) \}_{j=1,\ldots,n} $,
				\begin{align}\label{multidim, l=2}
					&\psi_{(Y^{(2)}_{i}(\bx))_{i \ge 1} | \{ Y^{(1)}_{j}(\bx) \}_{j} }(\bft) \\
					\nn&\qquad := \bbE\left[\left. \exp\left( i\sum_{i\in\cL} t_{i}Y^{(2)}_{i}(\bx) \right) \right| \{ Y^{(1)}_{j}(\bx) \}_{j} \right] \\
					\nn&\qquad = \bbE\left[ \exp\left( i\sum_{i\in\cL} B^{(2)}_{i}t_{i} \right) \right] \bbE\left[\left. \exp\left( i\frac{1}{a_{n}}\sum_{j=1}^{n} \sum_{i\in\cL} W^{(2)}_{ij}\phi(Y^{(1)}_{j}(\bx)) t_{i} \right)  \right| \{ Y^{(1)}_{j}(\bx) \}_{j}\right] \\
					\nn&\qquad = \psi_{B}(\bft) \prod_{j=1}^{n} \prod_{i\in\cL}\bbE\left[ \left. \exp\left( i\frac{1}{a_{n}} W^{(2)}_{ij}\phi(Y^{(1)}_{j}(\bx)) t_{i} \right)  \right| \{ Y^{(1)}_{j}(\bx) \}_{j}\right] \\
					\nn&\qquad = \psi_{B}(\bft) \prod_{j=1}^{n} \prod_{i\in\cL}\psi_{W}\left(\frac{\phi(Y^{(1)}_{j}(\bx))t_{i}}{a_{n}}\right) 
				\end{align}
				Taking the expectation over the randomness of $ \{ Y^{(1)}_{j}(\bx) \}_{j=1,\ldots,n} $,
				\begin{align*}
                                        \frac{\psi_{(Y^{(2)}_{i}(\bx))_{i \ge 1}}(\bft)}{\psi_{B}(\bft)}
					&= \int \prod_{j=1}^{n} \prod_{i\in\cL}\psi_{W}\left(\frac{\phi(y_{j})t_{i}}{a_{n}}\right) \, \bigotimes_{j= 1}^n \nu^{(1)}(dy_{j}) \\
					&= \left( \int \prod_{i\in\cL}\psi_{W}\left(\frac{\phi(y)t_{i}}{a_{n}}\right) \, \nu^{(1)}(dy) \right)^{n} .
				\end{align*}
                                Now since 
				\begin{align*}
					&\prod_{i\in\cL}\psi_{W}\left(\frac{\phi(y)t_{i}}{a_{n}}\right) \\
					&\qquad = 1 - c_\alpha\frac{b_n }{n} \sum_{i\in\cL} |\phi(y)t_{i}|^{\alpha} \frac{L_{0}\left( \frac{a_{n}}{|\phi(y)t_{i}|} \right)}{L_{0}(a_{n})} + o\left( \frac{b_n }{n} \sum_{i\in\cL} |\phi(y)t_{i}|^{\alpha} \frac{L_{0}\left( \frac{a_{n}}{|\phi(y)t_{i}|} \right)}{L_{0}(a_{n})} \right) ,
				\end{align*}
				it follows that
				\begin{align*}
                                        \frac{\psi_{(Y^{(2)}_{i}(\bx))_{i \ge 1}}(\bft)}{\psi_{B}(\bft)}
					&= \left( 1 - c_\alpha\frac{b_n }{n} \sum_{i\in\cL} \int |\phi(y)t_{i}|^{\alpha} \frac{L_{0}\left( \frac{a_{n}}{|\phi(y)t_{i}|} \right)}{L_{0}(a_{n})} \, \nu^{(1)}(dy) + o\left( \frac{b_n }{n} \right) \right)^{n} \\
					&\to \exp\left( -c_\alpha \sum_{i\in\cL} |t_{i}|^{\alpha} \int |\phi(y)|^{\alpha} \, \nu^{(1)}(dy) \right) \\
					&= \prod_{i\in\cL} \exp\left( -c_\alpha |t_{i}|^{\alpha} \int |\phi(y)|^{\alpha} \, \nu^{(1)}(dy) \right) .
				\end{align*}
				This proves the case $\ell=2$.
				
	                        \noindent{\bf Case $\ell>2$:}
				
					The remainder of the proof uses induction on the layer $\ell$, the base case being $\ell=2$ proved above.
				Let $ \ell > 2 $. Also, let $n=n_{\ell-1}$, $\alpha=\alpha_\ell$, $a_n=a_{n_{\ell-1}}(\ell)$, $\sigma_B=\sigma_{B^{(\ell)}}$,
				and $t\neq 0$.
				Then $ \{ Y^{(\ell-1)}_{j}(\bx) \}_{j=1,\ldots,n} $ is no longer i.i.d.; however, it is still exchangeable.
				By de Finetti's theorem (see Remark \ref{rem:1}), there exists a random probability measure 
				\begin{align}\label{def:xi}
                                        \xi^{(\ell-1)}(dy):=\xi^{(\ell-1)}(dy,\omega;
                                        \bn)
				\end{align}
				 such that given $\xi^{(\ell-1)}$, the random variables $Y^{(\ell-1)}_{j}(\bx), j=1,2,\ldots $ are i.i.d. with distribution $ \xi^{(\ell-1)}(dy,\omega)$ where $\omega\in\Omega$ is an element of the probability space.

				As before, we start by proving convergence of the marginal distribution.
				Taking the conditional expectation of \eqref{eq: conditional ch f}, given $\xi^{(\ell-1)}$, we have
				\begin{align*}
					\psi_{Y^{(\ell)}_{i}(\bx)|\xi^{(\ell-1)}}(t)
                                        &:= \bbE\left[\left. \psi_{Y^{(\ell)}_{i}(\bx) | \{ Y^{(\ell-1)}_{j}(\bx) \}_{j}}(t) \right| \xi^{(\ell-1)}\right] 
                                        \\
					&= e^{-\sigma_{B}^{\alpha} |t|^{\alpha}} \bbE\left[\left. \prod_{j=1}^{n} \psi_{W^{(\ell)}_{ij}}\left (\frac{\phi(Y^{(\ell-1)}_{j}(\bx))}{a_{n}}t\right ) \right|\xi^{(\ell-1)} \right] \\
					&= e^{-\sigma_{B}^{\alpha} |t|^{\alpha}} \left( \int \psi_W\left (\frac{\phi(y)}{a_{n}}t\right ) \, \xi^{(\ell-1)}(dy) \right)^{n}
				\end{align*}
                                where $\psi_W:=\psi_{W^{(\ell)}_{ij}}$ for some/any $i,j$.				
				Using Lemma \ref{lem: pitman approximation} and Lemma \ref{lem: pitman slowly varying bound} again, we get
				\begin{align}
                    &\left(\int \psi_W\left (\frac{\phi(y)t}{a_{n}}\right ) \, \xi^{(\ell-1)}(dy)\right)^{n} = {} \label{eq: ch f expansion integration general l} \\
                    &\quad \left(1 - c_\alpha\frac{b_n }{n} \int |\phi(y)t|^{\alpha}\frac{L_{0}\left( \frac{a_{n}}{|\phi(y)t|} \right)}{L_{0}(a_{n})} \, \xi^{(\ell-1)}(dy) + o\left ( \frac{b_n }{n} \int |\phi(y)t|^{\alpha}\frac{L_{0}\left( \frac{a_{n}}{|\phi(y)t|} \right)}{L_{0}(a_{n})} \, \xi^{(\ell-1)}(dy) \right )\right)^{n} . \nonumber
				\end{align}
                            Note that these are random integrals since $ \xi^{(\ell-1)}(dy) $ is random, whereas the corresponding integral in the case $ \ell=2 $ was deterministic. Also,  each integral on the right-hand side is finite almost surely since $ \phi $ is bounded.			
				By the induction hypothesis, the joint distribution of $ (Y^{(\ell-1)}_{i}(\bx))_{i \ge 1}$ converges weakly to the product measure $\bigotimes_{i \ge 1}  \mu_{\alpha_{\ell-1},\sigma_{\ell-1}}$.
				We claim that
				\begin{align}\label{eq: convergence claim}
					\int |\phi(y)t|^{\alpha}\frac{L_{0}\left( \frac{a_{n}}{|\phi(y)t|} \right)}{L_{0}(a_{n})} \, \xi^{(\ell-1)}(dy)
                                        \stackrel{p}{\to} |t|^{\alpha}\int |\phi(y)|^{\alpha} \, \mu_{\alpha_{\ell-1},\sigma_{\ell-1}}(dy) .
				\end{align}
				To see this, note that
				\begin{align}\label{eq: convergence claim bound}
                                        &\left| \int |\phi(y)t|^{\alpha}\frac{L_{0}\left( \frac{a_{n}}{|\phi(y)t|} \right)}{L_{0}(a_{n})} \, \xi^{(\ell-1)}(dy) - \int |\phi(y)t|^{\alpha} \, \mu_{\alpha_{\ell-1},\sigma_{\ell-1}}(dy) \right| \\ 
                                        \nonumber
                                        &\qquad \le \left| \int  |\phi(y)t|^{\alpha}  \frac{L_{0}\left( \frac{a_{n}}{|\phi(y)t|} \right)}{L_{0}(a_{n})} \, \xi^{(\ell-1)}(dy)-\int  |\phi(y)t|^{\alpha}  \frac{L_{0}\left( \frac{a_{n}}{|\phi(y)t|} \right)}{L_{0}(a_{n})} \, \mu_{\alpha_{\ell-1}, \sigma_{\ell-1}}(dy) \right| \\
                                        \nonumber
                                        & \qquad + \left| \int |\phi(y)t|^{\alpha} \frac{L_{0}\left( \frac{a_{n}}{|\phi(y)t|} \right)}{L_{0}(a_{n})} \, \mu_{\alpha_{\ell-1},\sigma_{\ell-1}}(dy) - \int |\phi(y)t|^{\alpha} \, \mu_{\alpha_{\ell-1},\sigma_{\ell-1}}(dy) \right| .
				\end{align}

				{First, consider the first term on the right-hand side of the above. By Corollary \ref{lem: convergence of random mixing measures}, the random measures $\xi^{(\ell-1)}$  {converge weakly, in probability,} to $\mu_{\alpha_{\ell-1}, \sigma_{\ell-1}}$ as $\bn\to\infty$ in the sense of \eqref{n to infinity}, where $\bn\in\N^{\ll}$. Also, by Lemma \ref{lem: pitman slowly varying bound}, we have
				\begin{equation}\label{eq: removing L inequality}
				    |\phi(y) t|^\alpha \frac{L_{0}\left( \frac{a_{n}}{|\phi(y)t|}\right)}{L_{0}(a_{n})} \leq b |\phi(y) t|^{\alpha \pm \eps}
                \end{equation}
                    for large $n$. For any subsequence $(\bn_j)_{j}$, there is a further subsequence $(\bn_{j_k})_k$ along which, $\omega$-a.s., $\xi^{(\ell-1)} $ converges weakly to $ \mu_{\alpha_{\ell-1},\sigma_{\ell-1}}$. To prove that the first term on the right-hand side of (\ref{eq: convergence claim bound}) converges in probability to $ 0 $, it is enough to show that it converges almost surely to $ 0 $ along each subsequence $(\bn_{j_k})_k$. Fix an $\omega$-realization of the random distributions $(\xi^{(\ell-1)}(dy,\omega;\bn))_{\bn \in \N^{\ll}}$ such that convergence along the subsequence $(\bn_{j_k})_k$ holds. Keeping $\omega$ fixed, view $g(y_{\bn})=|\phi(y_{\bn}) t|^{\alpha \pm \epsilon}$ as a random variable where the parameter $y_{\bn}$ is sampled from the distribution $\xi^{(\ell-1)}(dy,\omega;\bn)$. Since $ \phi $ is bounded, the family of these random variables is uniformly integrable. Since $ \xi^{(\ell-1)}(dy,\omega;\bn) $ converges weakly to $ \mu_{\alpha_{\ell-1},\sigma_{\ell-1}} $ along the subsequence, the Skorokhod representation and Vitali convergence theorem \cite[p. 94]{royden2010real} guarantee the convergence of the first term on the right-hand side of \eqref{eq: convergence claim bound} to $ 0 $ as $\bn$ tends to~$\infty$.}

				Now, for the second term, since $$\underset{n \to \infty}{\lim} |\phi(y)t|^{\alpha} \frac{L_{0}\left( \frac{a_{n}}{|\phi(y)t|} \right)}{L_{0}(a_{n})} = |\phi(y) t|^\alpha$$ for each $y$ and $ \phi $ is bounded, we can use dominated convergence via \eqref{eq: removing L inequality} to show that the second term on the right-hand side of \eqref{eq: convergence claim bound} also converges to zero, proving the claim.

				Having proved (\ref{eq: convergence claim}),
                         we have
				\begin{align*}
					&\left( 1 + \frac{1}{n} \left( -c_\alpha b_n  \int |\phi(y)t|^{\alpha} \frac{L_{0}\left( \frac{a_{n}}{|\phi(y)t|} \right)}{L_{0}(a_{n})} \, \xi^{(\ell-1)}(dy) + o( b_n  ) \right) \right)^{n} \\
                                        &\qquad  \stackrel{p}{\to} \exp\left( -c_\alpha |t|^{\alpha} \int |\phi(y)|^{\alpha} \, \mu_{\alpha_{\ell-1},\sigma_{\ell-1}}(dy) \right)
				\end{align*}
				and hence
				\begin{align*}
                                        \psi_{Y^{(\ell)}_i(\bx)| \xi^{(\ell-1)}}(t)
                                         \stackrel{p}{\to} e^{-\sigma_{B}^{\alpha}|t|^{\alpha}} \exp\left( -c_\alpha |t|^{\alpha} \int |\phi(y)|^{\alpha} \, \mu_{\alpha_{\ell-1},\sigma_{\ell-1}}(dy) \right) .
				\end{align*}
				Thus, the limiting distribution of $Y^{(\ell)}_{i}(\bx)$, given $\xi^{(\ell-1)}$,  is $\mu_{\alpha, \sigma_{\ell}}$ with
				\begin{align*}
                                        \sigma^{\alpha}_{\ell} &= \sigma_{B}^{\alpha} + c_\alpha \int |\phi(y)|^{\alpha} \, \mu_{\alpha,\sigma_{\ell-1}}(dy).
				\end{align*}
				Recall that characteristic functions are bounded by 1. Thus, by
				taking the expectation of both sides and using dominated convergence, we can conclude that the (unconditional) characteristic function converges to the same expression and thus the (unconditional)
						 distribution of $Y^{(\ell)}_{i}(\bx)$ converges weakly to $\mu_{\alpha, \sigma_{\ell}}$.

				Finally, we prove that the joint distribution converges weakly to the product $\bigotimes_{i\ge 1}\mu_{\alpha,\sigma_{\ell}} $. Let $ \cL \subset \bbN $ be a finite set and $ \bft = (t_{i})_{i\in\cL} $.
                                Conditionally on $ \{ Y^{(\ell-1)}_{j}(\bx) \}_{j=1,\ldots,n} $,
				\begin{align}\label{multidim general l}
                                        \psi_{(Y^{(\ell)}_{i}(\bx))_{i\ge 1} | \{ Y^{(\ell-1)}_{j}(\bx) \}_{j=1,\ldots,n} }(\bft)
                    &= \psi_{B}(\bft) \prod_{j=1}^{n} \prod_{i\in\cL}\psi_{W}\left(\frac{\phi(Y^{(\ell-1)}_{j}(\bx))t_{i}}{a_{n}}\right) .
				\end{align}
				Taking the expectation with respect to $ \{ Y^{(\ell-1)}_{j}(\bx) \}_{j=1,\ldots,n} $,
				\begin{align*}
					\frac{\psi_{(Y^{(\ell)}_{i}(\bx))_{i\ge 1}}(\bft)}{\psi_{B}(\bft)}
                    &= \bbE \int \prod_{j=1}^{n} \prod_{i\in\cL}\psi_{W}\left(\frac{\phi(y_{j})t_{i}}{a_{n}}\right) \, \bigotimes_{j\ge 1} \xi^{(\ell-1)}(dy_{j}) \\
                    &= \bbE \left( \int \prod_{i\in\cL}\psi_{W}\left(\frac{\phi(y)t_{i}}{a_{n}}\right) \, \xi^{(\ell-1)}(dy) \right)^{n} .
				\end{align*}
				Now since 
				\begin{align*}
                    \prod_{i\in\cL}\psi_{W}\left(\frac{\phi(y)t_{i}}{a_{n}}\right)
					\sim 1 - c_\alpha \frac{b_n }{n} \sum_{i\in\cL} |\phi(y)t_{i}|^{\alpha} \frac{L_{0}\left( \frac{a_{n}}{|\phi(y)t_{i}|} \right)}{L_{0}(a_{n})},
				\end{align*}
				a similar argument to that of convergence of the marginal distribution shows that
				\begin{align*}
					\frac{\psi(\bft)}{\psi_{B}(\bft)}
					&\sim \bbE \left( 1 - c_\alpha \frac{b_n }{n} \sum_{i\in\cL} \int |\phi(y)t_{i}|^{\alpha} \frac{L_{0}\left( \frac{a_{n}}{|\phi(y)t_{i}|} \right)}{L_{0}(a_{n})} \, \xi^{(\ell-1)}(dy) \right)^{n} \\
                                        &\to \exp\left( -c_\alpha  \sum_{i\in\cL} |t_{i}|^{\alpha} \int |\phi(y)|^{\alpha} \, \mu_{\alpha_{\ell-1},\sigma_{\ell-1}}(dy) \right) \\
                                        &= \prod_{i\in\cL} \exp\left( -c_\alpha  |t_{i}|^{\alpha} \int |\phi(y)|^{\alpha} \, \mu_{\alpha_{\ell-1},\sigma_{\ell-1}}(dy) \right)
				\end{align*}
				completing the proof.
			\end{proof} 

\section{Relaxing the Boundedness Assumption}\label{sec: Relaxing the Boundedness Assumption}

	As we mentioned earlier in Remark~\ref{rmk: boundedness of phi}, the boundedness assumption on $\phi$ can be relaxed,
        as long as it is done with care. To show the subtlety of our relaxation, we first present a counterexample where, for heavy-tailed initializations, we cannot use a function which grows linearly.
	\begin{remark}\label{rmk: relu counterexample}
		Consider the case where $ \phi = \operatorname{ReLU} $, $ \bbP( |W^{(\ell)}_{ij}| > t) = t^{-\alpha} $ for $ t \ge 1 $, $ 0<\alpha<2 $, and $ \sigma_{B} = 0 $.
		For an input $ \bx = (1,0,\ldots,0) \in \bbR^{I} $, we have
		\begin{align*}
			&Y^{(1)}_{i}(\bx) = W^{(1)}_{i1}, \\
			&Y^{(2)}_{i}(\bx) = \frac{1}{a_n} \sum_{j=1}^{n} W^{(2)}_{ij} W^{(1)}_{j1} \ind_{\{W^{(1)}_{j1}>0\}},\\
			& a_n = n^{1/\alpha}.
		\end{align*}
		Let us calculate the distribution function of $ W^{(2)}_{ij} W^{(1)}_{j1} \ind_{\{W^{(1)}_{j1}>0\}} $:
		For $ z \ge 1 $,
		\begin{align*}
			&\bbP(W^{(2)}_{ij} W^{(1)}_{j1} \ind_{\{W^{(1)}_{j1}>0\}} \le z) \\
			&\qquad = \bbP(W^{(1)}_{j1} \leq 0) + \int_{1}^{z} \frac{1}{2} \alpha w_{1}^{-\alpha-1} \bbP\left( W^{(2)}_{ij} \le \frac{z}{w_{1}} \right) \, dw_{1} + \int_{z}^{\infty} \frac{1}{2} \alpha w_{1}^{-\alpha-1} \bbP\left( W^{(2)}_{ij} \le -1 \right) \, dw_{1} \\
			&\qquad = 1 - \frac{1}{4} z^{-\alpha} - \frac{1}{4} \alpha z^{-\alpha} \log z .
		\end{align*}
		Similarly, for $ z \le -1 $,
		\begin{align*}
			\bbP(W^{(2)}_{ij} W^{(1)}_{j1} \ind_{\{W^{(1)}_{j1}>0\}} < z) = \frac{1}{4} \alpha (-z)^{-\alpha} \log(-z) + \frac{1}{4} (-z)^{-\alpha} .
		\end{align*}
		Thus,
		\begin{align*}
			\bbP(|W^{(2)}_{ij} W^{(1)}_{j1} \ind_{\{W^{(1)}_{j1}>0\}}| > z) = \frac{1}{2} z^{-\alpha} \left( 1 + \alpha \log z \right) .
		\end{align*}
		Let $ {\hat{a}}_n := \inf\{x\colon x^{-\alpha} (1+\alpha\log x)/2 \le n^{-1} \} $.
		Then, $ n \hat{a}_n^{-\alpha} (1+\alpha\log \hat{a}_n)/2 \to 1 $ as $ n \to \infty $, which leads to
		\begin{align*}
			\frac{\hat{a}_n}{n^{1/\alpha}} \sim ((1+\alpha\log \hat{a}_n)/2)^{1/\alpha} \to \infty
		\end{align*}
		when $ n $ is large.
		Thus, $ \hat{a}_n $ is of strictly larger order than $ n^{1/\alpha} $, so $ Y^{(2)}_{i}(\bx) $ cannot converge.
	\end{remark}

{Despite the above remark, there is still room to relax the boundedness assumption on $\phi$. Note that, in the proof of Theorem \ref{prop:main}, we used boundedness (in a critical way) to prove the claim \eqref{eq: convergence claim}. In particular, boundedness gave us that the family of random variables $ |\phi(y)|^{\alpha+\epsilon} $ with respect to the random distribution $ \xi^{(\ell-1)}(dy,\omega;\bn) $ is $ y $-uniformly integrable $ \omega $-almost surely. We directly make this into an assumption on $ \phi $ as follows: Let $n:=n_{\ell-2}$, and $a_n:=a_{n_{\ell-2}}(\ell-1)$.} Suppose
\begin{itemize}
	\item[(UI1)] for $\ell=2$, there exists $ \epsilon_{0}>0$ such that $ |\phi(Y^{(1)}_{j})|^{\alpha_{2}+\epsilon_{0}} $ is integrable;
	\item[(UI2)]  for $ \ell=3,\ldots,\ll+1 $, there exists $ \epsilon_{0}>0$ such that for any array $(c_{\bn,j})_{\bn,j}$ satisfying
	{\begin{align}\label{eq: c_{n,j} summable assumption}
		\sup_{\bn} \frac{1}{n}\sum_{j=1}^{n} |c_{\bn,j}|^{\alpha_{\ell-1}+\epsilon_{0}} < \infty,
	\end{align}}	
	we have uniform integrability of the family
	\begin{align}\label{eq: phi ui assumption}
	&\left\{ \left| \phi\left( \frac{1}{a_{n}} \sum_{j=1}^{n} c_{\bn,j} W^{(\ell-1)}_{j} \right)\right|^{\alpha_{\ell} + \epsilon_{0}}\right\}_{\bn}
	\end{align}
	over $\bn$.
\end{itemize}
If $ \phi $ is bounded, then the above is obviously satisfied.
It is not clear whether there is a simpler description of the family of functions that satisfy this assumption (see \cite{aldous1986classical}); however, let us argue now that this is general enough to recover the previous results of Gaussian weights or stable weights.

In \cite{matthews2018nngp} (as well as many other references), the authors considered Gaussian initializations with an activation function $ \phi $ satisfying the so-called polynomial envelop condition. That is, $ |\phi(y)| \le a + b|y|^{m} $ for some $ a,b > 0 $ and $ m \ge 1 $ and $ W \sim \cN(0,\sigma^{2}) $.
In this setting, we have $ a_{n} \sim \sigma\sqrt{n/2} $ and $ \alpha = 2 $ for all $\ell$, and $ c_{\bn,j} = c_{\bn,j}^{(\ell-2)}= \phi(Y^{(\ell-2)}_{j}(\bx;\bn)) $.
Conditioning on $(Y^{(\ell-2)}_j)_j$ and assuming that \eqref{eq: c_{n,j} summable assumption} holds a.s., let us show that $ \phi $ satisfying the polynomial envelope condition also satisfies our uniform integrability assumptions (UI1) and (UI2) a.s.
	For $ \ell=2 $, the distribution of
	\begin{align*}
		Y^{(1)}_{i} = \sum_{j=1}^{I} W^{(1)}_{ij} x_{j} + B^{(1)}_{i}
	\end{align*}
	is Gaussian, and thus $ |\phi(Y^{(1)}_{j})|^{2+\epsilon_{0}} \le C_0 + C_1 |Y^{(1)}_{j}|^{m(2+\epsilon_{0})}$ is integrable.
	For $ \ell \ge 3 $, note that 
	\begin{align*}
		S^{(\ell-1)}_{\bn} := \frac{1}{a_{n}} \sum_{j=1}^{n} c_{\bn,j} W^{(\ell-1)}_{j} \sim \cN\left(0, \frac{2}{n}\sum_{j=1}^{n} c_{\bn,j}^{2}\right) ,
	\end{align*}
	where the variance is uniformly bounded over $ \bn $ if we assume \eqref{eq: c_{n,j} summable assumption}.
For $ \theta > 1 $, the $ \nu:= m(2+\epsilon_{0})\theta $-th moment of $ S_{n} $ can be directly calculated, which is known to be
\begin{align*}
	 2^{\nu/2} \frac{1}{\sqrt{\pi}} \Gamma\(\frac{1 + \nu}{2}\)\( \frac{2}{n} \sum_{j=1}^{n} c_{\bn,j}^{2} \)^{\nu/2} .
\end{align*}
This is uniformly bounded over $ \bn $, and hence $ |\phi(S_{n})|^{2+\epsilon_{0}} $ is uniformly integrable over $ \bn $.
This shows that $ \phi $ satisfying the polynomial envelope condition meets (UI1) and (UI2) assuming~\eqref{eq: c_{n,j} summable assumption}.

In \cite{favaro2020stable}, the authors considered the case where $ W^{(\ell)} $ is a symmetric $ \alpha$-stable random variable with scale parameter $ \sigma_{\ell} $, i.e., with characteristic function $ e^{-\sigma_{\ell}^{\alpha}|t|^{\alpha}} $. They used the envelop condition $ |\phi(y)| \le a + b|y|^{\beta} $ where $ \beta < 1 $.
For the more general case where we have different $ \alpha_{\ell} $-stable weights for different layers $ \ell $, this envelop condition can be generalized to $ \beta < \min_{\ell\ge2} \alpha_{\ell-1}/\alpha_{\ell} $.
In this case, 
$ a_{n}^{\alpha_{\ell}} \sim 
(\sigma_{\ell}^{\alpha_{\ell}}n)/c_{\alpha_{\ell}}$ and $ c_{\bn,j} = c_{\bn,j}^{(\ell-2)} = \phi(Y^{(\ell-2)}_{j}(\bx;\bn)) $.
Again, conditioning on $(Y^{(\ell-2)}_{j})_j$  and assuming \eqref{eq: c_{n,j} summable assumption}, let us show that $ \phi $ under this generalized envelope condition satisfies the uniform integrability assumptions (UI1) and (UI2) above. 
For $ \ell=2 $, the distribution of 
\begin{align*}
	Y^{(1)}_{j} = \sum_{j=1}^{I} W^{(1)}_{ij} x_{j} + B^{(1)}_{i}
\end{align*}
is $ \alpha_{1} $-stable. By the condition on $\beta$, there are $ \delta $ and $ \epsilon_{0} $ satisfying $ \beta(\alpha_{2}+\epsilon_{0}) \le \alpha_{1} - \delta $ so that $$ |\phi(Y^{(1)}_{j})|^{\alpha_{2}+\epsilon} \le C_0 + C_1 |Y^{(1)}_{j}|^{\alpha_{1} - \delta} ,$$ which is integrable.
For $ \ell \ge 3 $, the distribution of $ S^{(\ell-1)}_{\bn} := a_{n}^{-1}\sum_{j} c_{\bn,j} W^{(\ell-1)}_{j} $ becomes a symmetric $ \alpha_{\ell-1} $-stable distribution with scale parameter
\begin{align*}
	\( \frac{c_{\alpha_{\ell-1}}}{n} \sum_{j=1}^{n} |c_{\bn,j}|^{\alpha_{\ell-1}} \)^{1/\alpha_{\ell-1}}
\end{align*}
which is uniformly bounded over $ \bn $ assuming \eqref{eq: c_{n,j} summable assumption}.
Since $ \beta < \min_{\ell\ge2} \alpha_{\ell-1}/\alpha_{\ell} $, it follows that, for some $ \theta>1 $, there exist small $ \epsilon_{0} > 0 $ 
and $\delta > 0$ such that
\begin{align*}
	\left| \phi\( S^{(\ell-1)}_{\bn} \) \right|^{(\alpha_{\ell}+\epsilon_{0})\theta}
	\le C_0 + C_1 \left| S^{(\ell-1)}_{\bn} \right|^{\beta(\alpha_{\ell} + \epsilon_{0})\theta}
        \le C_0 + C_1 \left| S^{(\ell-1)}_{\bn} \right|^{\alpha_{\ell-1} - \delta} .
\end{align*}
It is known that (for instance \cite{shanbhag1977certain}) the expectation of $ |S^{(\ell-1)}_{\bn}|^{\nu} $ with $ \nu < \alpha_{\ell-1} $ is
\begin{align*}
	K_{\nu} \( \frac{c_{\alpha_{\ell-1}}}{n} \sum_{j=1}^{n} |c_{\bn,j}|^{\alpha_{\ell-1}} \)^{\nu/\alpha_{\ell-1}}
\end{align*}
where $ K_{\nu} $ is a constant that depends only on $ \nu $ (and $ \alpha_{\ell-1} $). As this is bounded uniformly over $ \bn $, the family
\begin{align*}
	\left \{ \left| \phi\( S^{(\ell-1)}_{\bn} \) \right|^{\alpha_{\ell}+\epsilon_{0}} \right \}_{\bn}
\end{align*}
is uniformly integrable.
Thus our $ \phi $, under the generalized envelope condition, satisfies (UI1) and (UI2).

Let us now see that $ c_{\bn,j} $ satisfies condition \eqref{eq: c_{n,j} summable assumption} in both the Gaussian and symmetric stable case.
For $ \ell=3 $, $ c_{\bn,j} = \phi(Y^{(1)}_{j}) $ satisfies \eqref{eq: c_{n,j} summable assumption} by the strong law of large numbers since $ |\phi(Y^{(1)}_{j})|^{\alpha_{2}+\epsilon_{0}} $ is integrable.
For $ \ell > 3 $, an inductive argument shows that the family $ \{ |\phi(Y^{(\ell-2)}_{j})|^{\alpha_{\ell-1} + \epsilon_{0}} \}_{\bn} $ is uniformly integrable which leads to \eqref{eq: c_{n,j} summable assumption}. The details of this inductive argument are contained in the following proof.

\begin{proof}[Proof of Theorem~\ref{prop:main} under (UI1) and (UI2)]
	We return to the claim in (\ref{eq: convergence claim}) to see how conditions (UI1) and (UI2) are sufficient, even when $\phi$ is unbounded. We continue to let $n:=n_{\ell-2}$. Choose a sequence $\{(n,\bn)\}_n$, where $\bn=\bn(n)$ depends on $n$ and  $\bn\to\infty$ as $n\to\infty$ in the sense of \eqref{n to infinity}. Note that (i) to evaluate the limit as $\bn \to \infty$, it suffices to show the limit exists consistently for any choice of a sequence $\{\bn(n)\}_n$ that goes to infinity, and (ii) we can always pass to a subsequence (not depending on $\omega$) since we are concerned with convergence in probability. 
		Therefore, below we will show a.s. uniform integrability over some infinite subset of an arbitrary index set of the form $\{(n,\bn(n)): n \in \N\}$. 
	
	Let $a_n:=a_{n_{\ell-2}}(\ell-1)$.
	Proceeding as in (\ref{eq: convergence claim bound}) and (\ref{eq: removing L inequality}), we need to show that the family $ |\phi(y_{\bn})|^{\alpha+\epsilon} $ is uniformly integrable.
	Since $\{ a_{n}^{-1}\sum_{j} \phi(Y^{(\ell-2)}_{j}) W^{(\ell-1)}_{ij} \}_i$ is conditionally i.i.d. given $ \{Y^{(\ell-2)}_{j}\}_{j} $, the random distribution
	$ \xi^{(\ell-1)}(dy,\omega;\bn) $ is the law of $ a_{n}^{-1}\sum_{j} \phi(Y^{(\ell-2)}_{j}) W^{(\ell-1)}_{ij} $ given $ \{Y^{(\ell-2)}_{j}\}_{j} $ by the uniqueness of the directing random measure (\cite[Proposition 1.4]{kallenberg2006probabilistic}). Thus, by (UI2), it suffices to check that $ n^{-1}\sum_{j} |\phi(Y^{(\ell-2)}_{j})|^{\alpha_{\ell-1}+\epsilon_{0}} $ is uniformly bounded for $ \ell=3,\ldots,\ll+1 $.
	For $ \ell=3 $, since $ |\phi(Y^{(1)}_{j})|^{\alpha_{2}+\epsilon_{0}} $ is integrable by (UI1), 
	\[ \lim_{n\to \infty} \frac{1}{n} \sum_{j=1}^{n} |\phi(Y^{(1)}_{j})|^{\alpha_{2}+\epsilon_{0}} < \infty 
	\]
	by the strong law of large numbers and hence the normalized sums are almost surely bounded.
	For $ \ell > 3 $, we proceed inductively.
	By the inductive hypothesis, we have
	$$ \underset{\bn}{\sup}\frac{1}{n_{\ell-3}} \sum_{j=1}^{n_{\ell-3}} |\phi(Y^{(\ell-3)}_{j})|^{(\alpha_{\ell-2}+\epsilon_{0})(1+\eps')} < \infty$$
	by adjusting $\epsilon_0, \epsilon'>0$  appropriately. By (UI2), we have that the family $$\{|\phi(y_{\bn})|^{(\alpha_{\ell-1}+\eps_0)(1+\eps'')} : y_{\bn} \sim \xi^{(\ell-2)}(dy;\bn)\}$$
	is a.s. uniformly integrable for some $\eps''>0$. Since the $Y_j^{(\ell-2)}$'s are conditionally i.i.d. with common distribution $\xi^{(\ell-2)}(dy;\bn)$ given $\xi^{(\ell-2)}(dy,\omega;\bn)$, by Lemma \ref{lem: ui WLLN} we have that
	\begin{align*}
	& \P \Big( \Big| \frac{1}{n} \sum_{j=1}^{n} |\phi(Y^{(\ell-2)}_{j})|^{\alpha_{\ell-1}+\epsilon_{0}} - \int |\phi(y)|^{\alpha_{\ell-1}+\epsilon_0} \xi^{(\ell-2)} (dy;\bn) \Big| \geq \delta \ \Big| \ \xi^{(\ell-2)}\Big) \to 0
	\end{align*}
	almost surely. By the dominated convergence theorem we can take expectations on both sides to conclude that
	$$  \Big| \frac{1}{n} \sum_{j=1}^{n} |\phi(Y^{(\ell-2)}_{j})|^{\alpha_{\ell-1}+\epsilon_{0}} - \int |\phi(y)|^{\alpha_{\ell-1}+\epsilon_0} \xi^{(\ell-2)} (dy;\bn) \Big| \to 0$$
	in probability, so by passing to a subsequence the convergence holds for almost every $\omega$. Since $$ \underset{\bn}{\sup}\int |\phi(y)|^{\alpha_{\ell-1}+\epsilon_0} \xi^{(\ell-2)} (dy;\bn) < \infty$$
	almost surely, we have also that
	$$ \underset{\bn}{\sup} \frac{1}{n} \sum_{j=1}^{n} |\phi(Y^{(\ell-2)}_{j})|^{\alpha_{\ell-1}+\epsilon_{0}} < \infty$$
	almost surely, proving our claim.
\end{proof}

\section{Joint Convergence with Different Inputs}
\label{section:joint}

	In this section, we extend Theorem \ref{prop:main} to the joint distribution of $k$ different inputs. In this section, we show that the $ k $-dimensional vector $ (Y^{(\ell)}_{i}(\bx_{1};\bn), \ldots, Y^{(\ell)}_{i}(\bx_{k};\bn)) $ converges, and represent the limiting characteristic function via the spectral measure $ \Gamma_{\ell} $.
	
	For simplicity, we use the following notation:
	\begin{itemize}
		\item $\vec{\bx}= (\bx_{1}, \ldots, \bx_{k}) $ where $ \bx_{j} \in \bbR^{I} $.
		\item $ \ind = (1,\ldots,1) \in \bbR^{k} $.
		\item $ \bY^{(\ell)}_{i}(\vec{\bx};\bn) = (Y^{(\ell)}_{i}(\bx_{1};\bn), \ldots, Y^{(\ell)}_{i}(\bx_{k};\bn)) \in \bbR^{k} $, for $i\in\N$.
		\item $ \phi(\bY^{(\ell)}_{i}(\vec{\bx};\bn)) = (\phi(Y^{(\ell)}_{i}(\bx_{1};\bn)), \ldots, \phi(Y^{(\ell)}_{i}(\bx_{k};\bn))) \in \bbR^{k} $.
		\item $ \langle \cdot,\cdot \rangle $ denotes the standard inner product in $ \bbR^{k} $.
		\item 	For any given $j$, let the law of the $ k $-dimensional vector $ \bY^{(\ell)}_{j}(\vec{\bx} ) $ be denoted by $ \nu^{(\ell)}_{k} $ (which does not depend on $j$). Its projection onto the $ s $-th component $Y^{(\ell)}_{i}(\bx_{s};\bn)$ is denoted by $ \nu^{(\ell)}_{k,s} $ for $1\le s\le k$, and the projection onto the two coordinates, $i$-th and $j$-th, is denoted by $\nu^{(\ell)}_{k,ij}$. The limiting distribution of $\bY^{(\ell)}_{j}(\vec{\bx} )$ is denoted by $\mu_k^{(\ell)}$, and the projections are similarly denoted by $\mu_{k,s}^{(\ell)}$ and $\mu_{k, ij}^{(\ell)}.$ 
		\item A centered $k$-dimensional multivariate Gaussian with covariance matrix $M$ is denoted by $\mathcal{N}_k(M).$
				\item For $\alpha<2$, we denote the $ k $-dimensional symmetric $ \alpha $-stable distribution with spectral measure $ \Gamma $ by $ \text{\sas}_{k}(\Gamma)$.
For those not familiar with the spectral measure of a multivariate stable law, see Appendix \ref{appendix: stable}.
	\end{itemize}
	
	\begin{theorem}
	\label{theorem:multivariate-main}
		For each $ \ell = 2,\ldots,\ll+1 $, the joint distribution of the random variables $ (\bY^{(\ell)}_{i}(\vec{\bx};\bn))_{i\ge 1} $ converges weakly to, 
\begin{itemize}
\item for $\alpha_{\ell}<2$, $\mu_k^{(\ell)}=\bigotimes_{i\ge 1} \text{S}\alpha_{\ell}\text{S}_{k}(\Gamma_{\ell}) $ where $\Gamma_\ell$ is defined by
\begin{align}\label{gamma2}
			\Gamma_{2} = \norm{\sigma_{B^{(2)}}\ind}^{\alpha_{2}} \delta_{\frac{\ind}{\norm{\ind}}} + \int \norm{c_{\alpha_{2}}^{1/\alpha_{2}}(\phi(\by)) }^{\alpha_{2}} \, \delta_{\frac{\phi(\by)}{\norm{\phi(\by)}}} \, \nu^{(1)}_{k}(d\by) ,
		\end{align}
and
\begin{align}\label{gammal}
        \Gamma_{\ell} = \norm{\sigma_{B^{(\ell)}}\ind}^{\alpha_{\ell}} \delta_{\frac{\ind}{\norm{\ind}}} + \int \norm{c_{\alpha_{\ell}}^{1/\alpha_{\ell}}(\phi(\by)) }^{\alpha_{\ell}} \, \delta_{\frac{\phi(\by)}{\norm{\phi(\by)}}} \, \mu_k^{(\ell-1)}(d\by) ,
		\end{align}
\item for $\alpha_{\ell}=2$, $\mu_k^{(\ell)}=\bigotimes_{i\ge 1} \mathcal{N}_k(M_\ell)$, where
\begin{align}
	& (M_2)_{ii}= \E |B_i^{(2)}|^2+\frac{1}{2} \int |\phi(y)|^2 \, \nu_{k,i}^{(1)}(dy), \label{M2}
	\\ & (M_2)_{ij}= \frac{1}{2} \int \phi(y_1)\phi(y_2) \, \nu_{k,ij}^{(1)}(dy_1dy_2)  \nonumber
\end{align}
and
\begin{align}
	& (M_\ell)_{ii}= \E |B_i^{(\ell)}|^2+\frac{1}{2} \int |\phi(y)|^2 \, \mu_{k,i}^{(\ell-1)}(dy), \label{Ml}
	\\ & (M_\ell)_{ij}= \frac{1}{2} \int \phi(y_1)\phi(y_2) \, \mu_{k,ij}^{(\ell-1)}(dy_1dy_2). \nonumber
\end{align}
\end{itemize}

	\end{theorem}

	\begin{proof}
		Let $ \bft = (t_{1}, \ldots, t_{k}) $.
		We again start with the expression
		\begin{align}
            & \psi_{\bY^{(\ell)}_{i}(\vec{\bx} ) | \{ \bY^{(\ell-1)}_{j}(\vec{\bx} ) \}_{j\ge1}}(\bft) \label{eq: multi d process ch f} \\
            &\qquad {} = \bfE\left[\left. e^{i\langle \bft, \bY^{(\ell)}_{i}(\vec{\bx} ) \rangle} \right| \{ \bY^{(\ell-1)}_{j}(\vec{\bx} ) \}_{j\ge1} \right] \nonumber \\
            &\qquad {} = \bfE\left[\left. \exp\left({i\left\langle \bft, \frac{1}{a_{n}} \sum_{j=1}^{n} W^{(\ell)}_{ij} \phi(\bY^{(\ell-1)}_{j}(\vec{\bx} )) + B^{(\ell)}_{i} \ind \right\rangle}\right) \right| \{ \bY^{(\ell-1)}_{j}(\vec{\bx} ) \}_{j\ge1} \right] \nonumber \\
            &\qquad {} = \bfE e^{iB^{(\ell)}_{i} \langle \bft, \ind \rangle} \prod_{j=1}^{n} \bfE\left[\left. \exp\left({i\frac{1}{a_{n}}W^{(\ell)}_{ij}\left\langle \bft,  \phi(\bY^{(\ell-1)}_{j}(\vec{\bx} )) \right\rangle}\right) \right| \{ \bY^{(\ell-1)}_{j}(\vec{\bx} ) \}_{j\ge1} \right] \nonumber \\
            &\qquad {} = \psi_{B}(\langle \bft, \ind \rangle)\left(\psi_{W}\left( \frac{1}{a_{n}} \langle \bft, \phi(\bY^{(\ell-1)}_{j}(\vec{\bx} )) \rangle \right)\right)^{n}. \nonumber
		\end{align}
                Here $\psi_B$ and $\psi_W$ are characteristic
                functions of the random variables $B_i^{(\ell)}$
                and $W^{(\ell)}_{ij}$ for some/any $i,j$.

		\noindent{\bf Case $\ell=2$:}
		
		As before, let $n=n_1, \alpha=\alpha_2$, $a_n=a_{n_1}(2)$.
		As in Theorem \ref{prop:main}, $ (\bY^{(1)}_{j}(\vec{\bx} ))_{j \ge 1} $ is i.i.d, and thus
		\begin{align*}
			\psi_{\bY^{(\ell)}_{i}(\vec{\bx} )}(\bft)
			&= \psi_{B}(\langle \bft, \ind \rangle) \bfE \left(\psi_{W}\left( \frac{1}{a_{n}} \langle \bft, \phi(\bY^{(\ell-1)}_{j}(\vec{\bx} )) \rangle \right)\right)^{n} \\
			&= \psi_{B}(\langle \bft, \ind \rangle) \int \left(\psi_{W}\left( \frac{1}{a_{n}} \langle \bft, \phi(\by) \rangle \right)\right)^{n} \, \nu^{(1)}_{k}(d\by).
		\end{align*}
		As before,
		\begin{align*}
			&\left(\psi_{W}\left( \frac{1}{a_{n}} \langle \bft, \phi(\by) \rangle \right)\right)^{n} \\
			&\qquad = \left(1 - c_{\alpha} \frac{b_{n}}{n} \left| \langle \bft, \phi(\by) \rangle\right|^{\alpha} \frac{L_0\left(\frac{a_{n}}{|\langle \bft, \phi(\by) \rangle|} \right)}{L_0(a_{n})} + o\left(  \frac{b_{n}}{n} \left| \langle \bft, \phi(\by) \rangle\right|^{\alpha} \frac{L_0\left(\frac{a_{n}}{|\langle \bft, \phi(\by) \rangle|} \right)}{L_0(a_{n})} \right)  \right)^{n} .
		\end{align*}
		The main calculation needed to extend the proof of Theorem \ref{prop:main} to the situation involving $\vec{\bx}$ is as follows.
		Assuming the uniform integrability in Section~\ref{sec: Relaxing the Boundedness Assumption}, we have, for some $b > 0$ and $0<\epsilon<\epsilon_{0}$,
                \begin{align} 
                        \label{eq: multi d bound}
				\int \left| \langle \bft, \phi(\by) \rangle\right|^{\alpha} \frac{L_0\left(\frac{a_{n}}{|\langle \bft, \phi(\by) \rangle|} \right)}{L_0(a_{n})} \, \nu^{(1)}_{k}(d\by)
                                &{} \leq b \int \left| \langle \bft, \phi(\by) \rangle \right|^{\alpha \pm \epsilon} \, \nu^{(1)}_{k}(d\by) \\
                                \nonumber
                                &{} =\int b \left| \sum_{s=1}^{k} t_{s} \phi(y_{s}) \right|^{\alpha \pm \epsilon} \, \nu^{(1)}_{k}(d\by) \\
                                \nonumber
                                &{} \le \int b c_{k} \sum_{s=1}^{k} |t_{s} \phi(y_{s})|^{\alpha \pm \epsilon} \, \nu^{(1)}_{k}(d\by) \\
                                \nonumber
                                &{} = b c_{k} \sum_{s=1}^{k} \int |t_{s} \phi(y_{s})|^{\alpha \pm \epsilon} \, \nu^{(1)}_{k,s}(dy) < \infty .
                \end{align}
		It thus follows that
		\begin{align*}
			&\int \left| \langle \bft, \phi(\by) \rangle\right|^{\alpha} \frac{L_0\left(\frac{a_{n}}{|\langle \bft, \phi(\by) \rangle|} \right)}{L_0(a_{n})} \, \nu^{(1)}_{k}(d\by)
			\to \int \left| \langle \bft, \phi(\by) \rangle\right|^{\alpha} \, \nu^{(1)}_{k}(d\by) ,\quad\text{and} \\
			&\int o\left(\frac{b_{n}}{n} \left| \langle \bft, \phi(\by) \rangle\right|^{\alpha} \frac{L_0\left(\frac{a_{n}}{|\langle \bft, \phi(\by) \rangle|} \right)}{L_0(a_{n})}\right)  \, \nu^{(1)}_{k}(d\by)
			= o\left(\frac{b_{n}}{n}\right) .
		\end{align*}
		Therefore,
		\begin{align}\label{eq:cf multi}
			\psi_{B}(\langle \bft, \ind \rangle) 
			\bfE \left(\psi_{W}\left( \frac{1}{a_{n}} \langle \bft, \phi(\by) \rangle \right)\right)^{n}
			\to \exp\left( -\sigma_{B}^{\alpha} | \langle \bft, \ind \rangle |^{\alpha} \right) \exp\left( -c_{\alpha} \int \left| \langle \bft, \phi(\by) \rangle\right|^{\alpha} \, \nu^{(1)}_{k}(d\by) \right) .
		\end{align}
		Let $\norm{\cdot}$ denote the standard Euclidean norm. Observe that for $\alpha<2$,
		\begin{align*}
			c_{\alpha} \left| \langle \bft, \phi(\by) \rangle\right|^{\alpha}
			= \int_{S^{k-1}} |\langle \bft, \bfs \rangle|^{\alpha} \norm{c_{\alpha}^{1/\alpha}(\phi(\by)) }^{\alpha} \, \delta_{\frac{\phi(\by)}{\norm{\phi(\by)}}}(d\bfs) .
		\end{align*}
		Thus, by Theorem \ref{spectral measure}, we have the convergence $ \bY^{(\ell)}_{i}(\vec{\bx};\bn) \stackrel{w}{\to} \text{\sas}_{k}(\Gamma_{2}) $ where $\Gamma_2$ is defined by \eqref{gamma2}.

For $\alpha=2$, we have $$ \exp \Big(-c_\alpha \int |\langle \bft, \phi(\by) \rangle|^2 \nu_k^{(1)}(d \by) \Big) = \exp ( -\frac{1}{2} \langle \bft, M_2\bft \rangle)$$
where $M_2$ is given by \eqref{M2}, which is equal to the characteristic function of $\mathcal{N}(M_2)$.

		Extending the calculations in \eqref{multidim, l=2},
		the convergence $ (\bY^{(\ell)}_{i}(\vec{\bx};\bn))_{i\ge 1} \stackrel{w}{\to} \bigotimes_{i\ge 1} \text{\sas}_{k}(\Gamma_{2}) $ follows similarly.
		
		\noindent{\bf Case $\ell>2$:}
		
		Similar to \eqref{def:xi}, let $\xi^{(\ell-1)}(d\by,\omega)$ be a random distribution such that, given $\xi^{(\ell-1)}$, the random vectors $\bY^{(\ell-1)}_{j}(\vec{\bx}), j=1,2,\ldots $ are i.i.d. with distribution $ \xi^{(\ell-1)}(d\by)$.
		
		Taking the conditional expectation of $ (\ref{eq: multi d process ch f}) $ given $ \xi^{(\ell-1)} $, we get
		\begin{align*}
			\psi_{\bY^{(\ell)}_i|\xi^{(\ell-1)}} (\bft)
			= \psi_{B}(\langle \bft, \ind \rangle) \E\left[\left. \left( \int \psi_{W}\left( \frac{1}{a_{n}}\langle \bft, \phi(\by) \rangle \right) \xi^{(\ell-1)}(d\by) \right)^{n} \,\right|\, \xi^{(\ell-1)} \right] 
		\end{align*}
                for any $i$. Here,
		\begin{align*}
			\int \psi_{W}\left( \frac{1}{a_{n}}\langle \bft, \phi(\by) \rangle \right) \xi^{(\ell-1)}(d\by)
			\sim 1 - c_{\alpha} \frac{b_{n}}{n} \int |\langle \bft, \phi(\by) \rangle|^{\alpha} \frac{L_0(\frac{a_{n}}{|\langle \bft, \phi(\by) \rangle|})}{L_0(a_{n})} \xi^{(\ell-1)}(d\by) .
		\end{align*}
		From the induction hypothesis, $ (\bY^{(\ell-1)}_{i}(\vec{\bfx}))_{i\ge1} $ converges weakly either to $ \bigotimes_{i\ge1} \text{\sas}(\Gamma_{\ell-1}) $ or to $ \bigotimes_{i\ge1} \mathcal{N}_k(M_\ell) $.
		We claim that
		\begin{align*} 
			\int |\langle \bft, \phi(\by) \rangle|^{\alpha} \frac{L_0(\frac{a_{n}}{|\langle \bft, \phi(\by) \rangle|})}{L_0(a_{n})} \xi^{(\ell-1)}(d\by)
			\stackrel{p}{\to} \int |\langle \bft, \phi(\by) \rangle|^{\alpha} \mu_k^{(\ell-1)}(d\by) .
		\end{align*}
		To see this, note that
		\begin{align}
			&\left| \int |\langle \bft, \phi(\by) \rangle|^{\alpha} \frac{L_0(\frac{a_{n}}{|\langle \bft, \phi(\by) \rangle|})}{L_0(a_{n})} \xi^{(\ell-1)}(d\by) - \int |\langle \bft, \phi(\by) \rangle|^{\alpha} \mu_k^{(\ell-1)}(d\by) \right| \label{eq: multi d claim difference} \\
			&\qquad \le \left| \int |\langle \bft, \phi(\by) \rangle|^{\alpha} \frac{L_0(\frac{a_{n}}{|\langle \bft, \phi(\by) \rangle|})}{L_0(a_{n})} \xi^{(\ell-1)}(d\by) - \int |\langle \bft, \phi(\by) \rangle|^{\alpha} \frac{L_0(\frac{a_{n}}{|\langle \bft, \phi(\by) \rangle|})}{L_0(a_{n})} \mu_k^{(\ell-1)}(d\by) \right| \nonumber \\
			&\qquad +\left| \int |\langle \bft, \phi(\by) \rangle|^{\alpha} \frac{L_0(\frac{a_{n}}{|\langle \bft, \phi(\by) \rangle|})}{L_0(a_{n})} \mu_k^{(\ell-1)}(d\by) - \int |\langle \bft, \phi(\by) \rangle|^{\alpha} \mu_k^{(\ell-1)}(d\by) \right| . \nonumber
		\end{align}
		Now, the uniform integrability assumption in Section~\ref{sec: Relaxing the Boundedness Assumption} combined with (\ref{eq: multi d bound}) shows that
		\begin{align*}
			\left| \langle \bft, \phi(\by) \rangle\right|^{\alpha} \frac{L_0\left(\frac{a_{n}}{|\langle \bft, \phi(\by) \rangle|} \right)}{L_0(a_{n})}
		\end{align*}
		is uniformly integrable with respect to the family $ (\xi^{(\ell-1)})_{n} $, and thus the first term on the right-hand-side of (\ref{eq: multi d claim difference}) converges in probability to zero.
		Also, from (\ref{eq: multi d bound}) and the fact that
		\begin{align*}
			\lim_{n\to \infty} \left| \langle \bft, \phi(\by) \rangle\right|^{\alpha} \frac{L_0\left(\frac{a_{n}}{|\langle \bft, \phi(\by) \rangle|} \right)}{L_0(a_{n})} = \left| \langle \bft, \phi(\by) \rangle\right|^{\alpha}
		\end{align*}
		for each $ \by $, dominated convergence gives us convergence to $ 0 $ of the second term.
		Therefore,
		\begin{align*}
			\left( \int \psi_{W}\left( \frac{1}{a_{n}}\langle \bft, \phi(\by) \rangle \right) \xi^{(\ell-1)}(d\by) \right)^{n}
			\stackrel{p}{\to} \exp\left( -c_{\alpha} \int |\langle \bft, \phi(\by) \rangle|^{\alpha} \mu_k^{(\ell-1)}(d\by) \right)
		\end{align*}
		and consequently,
		\begin{align*}
			\psi_{\bY^{(\ell)}_i|\xi^{(\ell-1)}} (\bft)
			\stackrel{p}{\to} \psi_{B}(\langle \bft, \ind \rangle) \exp\left( -c_{\alpha} \int |\langle \bft, \phi(\by) \rangle|^{\alpha} \mu_k^{(\ell-1)}(d\by) \right) .
		\end{align*}
		Finally, noting that the characteristic function is bounded by $ 1 $ and using dominated convergence, we get
		\begin{align*}
			\psi_{\bY^{(\ell)}_i} (\bft)
			\stackrel{p}{\to} \psi_{B}(\langle \bft, \ind \rangle) \exp\left( -c_{\alpha} \int |\langle \bft, \phi(\by) \rangle|^{\alpha} \mu_k^{(\ell-1)}(d\by) \right),
		\end{align*}
		where the right-hand side is the characteristic function of $\text{\sas}_{k}(\Gamma_{\ell}) $ (or $\mathcal{N}_k(M_\ell)$ for $\alpha=2$), where $\Gamma_\ell$ and $M_\ell$ are given by \eqref{gammal} and \eqref{Ml}, respectively.

		The proof of $ (\bY^{(\ell)}_{i}(\vec{\bfx};\bn))_{i\ge1} \stackrel{w}{\to} \bigotimes_{i\ge1} \text{\sas}_{k}(\Gamma_{\ell}) $ (or  $\bigotimes_{i\ge 1} \mathcal{N}_k(M_\ell)$ in the case $\alpha=2$) follows similarly to the calculations following \eqref{multidim general l}.
	\end{proof}

\appendix
		
	\section{Auxiliary Lemmas}
		\begin{lemma}\label{lem: tilde L}
			If $ L $ is slowly varying, then
			\begin{align*}
			\widetilde{L}(x)
			= \int_{0}^{x} t^{-1}L(t) \, dt
			\end{align*}
			is also slowly varying.
		\end{lemma}
		\begin{proof}
			If $ \widetilde{L} $ is bounded, then since $ \widetilde{L} $ is increasing, $ \widetilde{L}(x) $ converges as $ x \to \infty $.
			Thus $ \widetilde{L} $ is slowly varying.
			If $ \widetilde{L} $ is not bounded, then by L'H\^opital's rule,
			\begin{align*}
			\lim_{x \to \infty}\frac{\widetilde{L}(\lambda x)}{\widetilde{L}(x)}
			= \lim_{x \to \infty}\frac{\int_{0}^{x} y^{-1}L(\lambda y) \, dy}{\int_{0}^{x} y^{-1}L(y) \, dy}
			= \lim_{x \to \infty}\frac{L(\lambda x)}{L(x)}
			= 1 .
			\end{align*}
		\end{proof}
		
		The next four lemmas are standard results for which we give references for their proofs.
		In particular, the next lemma is a standard result concerning the characteristic function of heavy-tailed distributions \cite[Theorem 1 and Theorem 3]{pitman1968behaviour} (see also \cite[Eq. 3.8.2]{durrett2019probability}).
		
		\begin{lemma}\label{lem: pitman approximation}
			If $ W $ is a symmetric random variable with tail probability $ \P(|W| > t) = t^{-\alpha}L(t) $ where $ 0 < \alpha \le 2 $ and $ L $ is slowly-varying, then the characteristic function $ \psi_{W}(t) $ of $ W$ satisfies
			\begin{align*}
				\psi_{W}(t) = 1- c_\alpha |t|^{\alpha}L\left(\frac{1}{|t|}\right) + o\left (|t|^{\alpha}L\left(\frac{1}{|t|}\right)\right ), \quad t \to 0
			\end{align*}
			where
			\begin{align*}
				c_\alpha  = \lim_{M \to \infty} \int_{0}^{M} \frac{\sin u}{u^{\alpha}} \, du
				= \frac{\pi/2}{\Gamma(\alpha) \sin(\pi\alpha/2)}
			\end{align*}
			for $ \alpha <2 $, and
			\begin{align*}
				\psi_{W}(t) = 1- |t|^{2}\widetilde{L}\left(\frac{1}{|t|}\right) + o\left (|t|^{2}\widetilde{L}\left(\frac{1}{|t|}\right)\right ), \quad t \to 0
			\end{align*}
			for $ \alpha =2 $
			where
			\begin{align*}
				\widetilde{L}(x)
				= \int_{0}^{x} y \P(|W| > y) \, dy
				= \int_{0}^{x} y^{-1}L(y) \, dy .
			\end{align*}
		\end{lemma}
		
		We next state a standard result about slowly varying functions \cite[VIII.8 Lemma 2]{feller1971introduction}.
		
		\begin{lemma}\label{lem: feller uniform convergence}
			If $ L $ is slowly varying, then for any fixed $ \epsilon>0$ and all sufficiently large $ x $,
			\begin{align*}
				x^{-\epsilon} < L(x) < x^{\epsilon}.
			\end{align*}
			Moreover, the convergence
			\begin{align*}
				\frac{L(tx)}{L(t)} \to 1
			\end{align*}
			as $ t \to \infty $ is uniform in finite intervals $ 0 < a < x < b $.
		\end{lemma}
		 An easy corollary of the above lemma is the following result, which we single out for convenience \cite[Lemma 2]{pitman1968behaviour}.
		\begin{lemma}\label{lem: pitman slowly varying bound}
	If $ G(t) = t^{-\alpha}L(t) $ where $ \alpha \ge 0 $ and $ L $ is slowly-varying, then for any given positive $ \epsilon $ and $ c $, there exist $ a $ and $ b $ such that
	\begin{align*}
	&\frac{G(\lambda t)}{G(t)} < \frac{b}{\lambda^{\alpha+\epsilon}} \quad \text{for } t \ge a, 0<\lambda \le c \\
	&\frac{G(\lambda t)}{G(t)} < \frac{b}{\lambda^{\alpha-\epsilon}} \quad \text{for } t \ge a, \lambda \ge c .
	\end{align*}
In particular, for sufficiently large $ t>0 $, we have
\begin{equation*} 
\frac{G(\lambda t)}{G(t)} \leq b {\({1/\lambda}\)^{\alpha \pm \epsilon}}
\end{equation*}
for all $\lambda>0$, where we define $x^{\alpha \pm \epsilon}:=\max\(x^{\alpha +\epsilon},x^{\alpha -\epsilon}\)$.
\end{lemma}		
		
		The next lemma regards the convolution of distributions with regularly varying tails \cite[VIII.8 Proposition]{feller1971introduction}.
		
		\begin{lemma}\label{lem: sum of heavy tailed}
			For two distributions $ F_{1} $ and $ F_{2} $ such that as $ x \to \infty $
			\begin{align*}
				1-F_{i}(x) = x^{-\alpha}L_{i}(x)
			\end{align*}
			with $ L_{i} $ slowly varying, the convolution $ G = F_{1} * F_{2} $ has a regularly varying tail such that
			\begin{align*}
				1-G(x) \sim x^{-\alpha}(L_{1}(x) + L_{2}(x)) .
			\end{align*}
		\end{lemma}

Recall that de Finetti's theorem tell us that if a sequence $\bX=(X_i)_{i\in\N}\in\R^\N$ is exchangeable then
\begin{align}\label{eq:definetti}
\P(\bX\in A) = \int \nu^{\otimes \bbN}(A) \,\pi(d\nu)
\end{align}
 for some $\pi$ which is a probability measure on the space of probability measures $\Pr(\R)$. The measure $\pi$ is sometimes called the {\it mixing measure}. Our final lemma characterizes the convergence of exchangeable sequences by convergence of their respective mixing measures. It is intuitively clear. However, its proof is not completely trivial. As far as we know, this lemma has not appeared in the literature before.

\begin{lemma}\label{lem:mixing measure}
\label{lem:mixing-measure-convergence} For each $j\in\N\cup\{\infty\}$, let $\bX^{(j)}=(X^{(j)}_i)_{i\in\N}$ be an infinite exchangeable sequence of random variables with values in $\R$ (or more generally, a Borel space).
Let $\pi_j$ be the mixing measure on $\Pr(\R)$ corresponding to $\bX^{(j)}$, from \eqref{eq:definetti}. Then the family $(\bX^{(j)})_{j\in\N}$ converges in distribution to $\bX^{(\ff)}$ if and only if the family  $(\pi_j)_{j\in\N}$ converges in the weak topology on $\Pr(\Pr(\R))$ to $\pi_\ff$.
\end{lemma}
In the lemma, the topology on $\Pr(\Pr(\R))$ is formed by applying the weak-topology construction twice. We first construct the weak topology on $\Pr(\R)$. Then, we apply the weak-topology construction again this time using $\Pr(\R)$, instead of $\R$.

	In the proof of Theorem \ref{prop:main}, we use the special case when the limiting sequence $\bX^{(\ff)}$ is a sequence of i.i.d. random variables. In that case, by \eqref{eq:definetti}, it must be that $\pi_\ff$ concentrates on a single element $\nu\in\Pr(\R)$, i.e. it is a point mass, $\pi_\ff=\delta_\nu$, for some $\nu\in\Pr(\R)$.
	
	More specifically, we have the following corollary.
			\begin{corollary}\label{lem: convergence of random mixing measures}
                                In the setting of Theorem \ref{prop:main}, the joint distribution of the exchangeable sequence $ (Y^{(\ell-1)}_{i}(\bx))_{i \ge 1}$ converges weakly to the product measure $\bigotimes_{i \ge 1}  \mu_{\alpha,\sigma_{\ell-1}}$ as $\min(n_1,\ldots,n_{\ll})$ tends to $\ff$ if and only if 
                                the random probability measures $ (\xi^{(\ell-1)}(dy,\omega;\bn))_{\bn\in\N^{\ll}}$ defined in \eqref{def:xi}  converge weakly, in probability, to the deterministic probability measure $ \mu_{\alpha,\sigma_{\ell-1}}$.

	\end{corollary}

\begin{proof}[Proof of Lemma \ref{lem:mixing measure}]

First suppose $(\pi_j)_{j\in\N}$ converges to $\pi_\ff$. We want to show that $(\bX^{(j)})_{j\in\N}$ converges in distribution to $\bX^{(\ff)}$. By \cite[Theorem 4.29]{kallenberg2002foundations}, convergence in distribution of a sequence of random variables is equivalent to showing that for every $m > 0$ and all bounded continuous functions $f_1,\ldots,f_m$, we have $$\bbE\Bigr[f_1(X_1^{(j)})\cdots f_m(X_m^{(j)})\Bigr]\to\bbE\Bigr[f_1(X_1^{(\ff)})\cdots f_m(X_m^{(\ff)})\Bigr]$$ as $j \to \ff$.  Rewriting the above using \eqref{eq:definetti} we must show that as $j \to \ff$,
\[
\int_{\Pr(\R)} \left(\int_{\R^m} \prod_{i = 1}^m f_i(x_i)\, \nu^{\otimes m}(d\bx)\right) \pi_{j}(d\nu) 
\longrightarrow
\int_{\Pr(\R)} \left(\int_{\R^m} \prod_{i = 1}^m f_i(x_i)\,\nu^{\otimes m}(d\bx)\right) \pi_\ff(d\nu).
\]
But this follows since $\nu\mapsto \int_{\R^m} \prod_{i = 1}^m f_i(x_i)\, \nu^{\otimes m}(d\bx)$ is a bounded continuous function on $\Pr(\R)$ with respect to the weak topology.

We now prove the reverse direction. We assume $(\bX^{(j)})_{j\in\N}$ converges in distribution to $\bX^{(\ff)}$ and must show that $(\pi_j)_{j\in\N}$ converges to $\pi_\ff$. 

In order to show this we first claim that the family $(\pi_j)_{j\in\N}$ is tight. By \cite[Theorem 4.10]{kallenberg2017random} (see also \cite[Theorem A.6]{ghosal2017fundamentals}), such tightness is equivalent to the tightness of the expected measures $$\left(\int \nu^{\otimes \bbN}\,\pi_j(d\nu)\right)_{j\in\N}.$$ But these are just the distributions of the family  $(\bX^{(j)})_{j\in\N}$ which we have assumed converges in distribution. Hence, its distributions are tight.

Let us return now to proving $(\pi_j)_{j\in\N}$ converges to $\pi_\ff$. Suppose to the contrary that this is not the case. Since the family $(\pi_j)_{j\in\N}$ is tight, by Prokhorov's theorem there must be another limit point of this family, $\tilde\pi\neq\pi_\ff$, and a subsequence $(j_n)_{n\in\N}$ such that 
$$\pi_{j_n} \weak \tilde\pi$$
as $n \to \ff$. By the first part of our proof, this implies that  $(\bX^{(j_n)})_{n\in\N}$ converges in distribution to an exchangeable sequence with distribution $\int \nu^{\otimes \bbN}\,\tilde\pi(d\nu)$. However, by assumption  we have that $(\bX^{(j)})_{j\in\N}$ converges in distribution to $\bX^{(\ff)}$ which has distribution $\int \nu^{\otimes \bbN}\,\pi_\ff(d\nu)$. Thus, it must be that
$$\int \nu^{\otimes \bbN}\,\tilde\pi(d\nu) = \int \nu^{\otimes \bbN}\,\pi_\ff(d\nu).$$
But Proposition 1.4 in \cite{kallenberg2006probabilistic} tells us that the measure $\pi$ in \eqref{eq:definetti} is unique contradicting $\tilde\pi\neq\pi_\ff$. Thus, it must be that $(\pi_j)_{j\in\N}$ converges to $\pi_\ff$.
\end{proof}

\begin{lemma}\label{lem: ui WLLN} Let $\{X_{kn}: k \in \N\}$ be i.i.d. with $\E X_{1n}=0$ for each $n \in \N$. If the family $\{|X_{1n}|^p : n \in \N \}$ is uniformly integrable for some $p>1$, then as $n \to \infty$, we have $$S_n:=\frac{1}{n}\sum_{k=1}^n X_{kn} \to 0$$ in probability.
\end{lemma}
\begin{proof} For $M>0$, let $$Y_{kn}:= X_{kn} \mathbf{1}_{[|X_{kn}| \leq M]}-\E \big( X_{kn} \mathbf{1}_{[|X_{kn}| \leq M]} \big), \ Z_{kn}:= X_{kn} \mathbf{1}_{[|X_{kn}| >  M]}-\E \big( X_{kn} \mathbf{1}_{[|X_{kn}| > M]} \big),$$ $$T_{n}:=\frac{1}{n}\sum_{k=1}^n Y_{kn}, \ U_{n}:=\frac{1}{n}\sum_{k=1}^n Z_{kn} .$$ By Markov's inequality,
$$ \P \big( |T_n| \geq \delta \big) \leq \frac{\mathbf{Var}\, Y_{1n}}{n \delta^2} \leq \frac{4 M^2}{n \delta^2},$$
and
$$ \P \big( |U_n| \geq \delta \big) \leq \frac{\E|U_n|^p}{\delta^{p}} \leq \frac{\E |Z_{1n}|^p}{\delta^p}.$$
Thus, we have
$$ \underset{n \to \infty}{{}\limsup {}}\P\big( |S_n| \geq 2 \delta \big) \leq \frac{1}{\delta^p} \underset{n}{{} \sup{} }\E |Z_{1n}|^p.$$
By the uniform integrability assumption, the right-hand side can be made arbitrarily small by increasing $M$.
\end{proof}

\section{Multivariate Stable Laws}\label{appendix: stable}

This section contains some basic definition and properties of multivariate stable distributions
which may help familiarize some readers. The material in this section comes from the monograph \cite{samorodnitsky1994stable} and also \cite{kuelbs1973representation}.

\begin{definition}\label{def: stable} A probability measure $\mu$ on $\R^k$ is said to be \textbf{(jointly) stable} if for all $a,b \in \R$ and two independent random variables $X$ and $Y$ with distribution $\mu$, there exist $c \in \R$ and $v \in \R^k$ such that
$$aX+bY \buildrel d \over = cX+v.$$
If $\mu$ is symmetric, then it is said to be \textbf{symmetric stable}.
\end{definition}

Similar to the one-dimensional case, there exists a constant $\alpha \in (0,2]$ such that $c^\alpha=a^\alpha+b^\alpha$ for all $a,b$, which we call the \textbf{index of stability}. The distribution $\mu$ is multivariate Gaussian in the case $\alpha=2$.

\begin{theorem}\label{spectral measure} Let $\alpha \in (0,2)$. A random variable $\bX$ taking values in $\R^k$ is symmetric stable if and only if there exists a finite symmetric measure $\Gamma$ on the unit sphere $S_{k-1}=\{ x \in \R^k : |x|=1 \}$ such that
	\begin{align}\label{eq: ch f of stable}
		\E \exp \Big( i\langle\mathbf{t},\bX\rangle \Big) = \exp \Big( -\int_{S_{k-1}} |\langle\mathbf{t},\mathbf{s}\rangle|^\alpha \Gamma(d \mathbf{s}) \Big)
	\end{align}
for all $\mathbf{t} \in \R^k.$ The measure $\Gamma$ is called the \textbf{spectral measure} of $\bX$, and the distribution is denoted as $\text{\sas}_{k}(\Gamma)$.
\end{theorem}

In the case $k=1$, the measure $\Gamma$ is always of the form $c \delta_1 + c \delta_{-1}$. Thus, the characteristic function reduces to the familiar form
$$ \E e^{itX} = e^{-|\sigma t|^\alpha}.$$

\subsection*{Acknowledgments.}{\small We thank François Caron and Juho Lee for sugggesting the paper \cite{favaro2020stable} to us. PJ and HL were funded in part by the National
	Research Foundation of Korea (NRF) grant NRF-2017R1A2B2001952. PJ, HL, and JL were funded in part by  the National
	Research Foundation of Korea (NRF) grant 
	NRF-2019R1A5A1028324. HY was supported by the Engineering Research Center Program through the National Research Foundation of Korea (NRF) funded by the Korean Government MSIT (NRF-2018R1A5A1059921), and also by Next-Generation Information Computing Development Program through the National Research Foundation of Korea (NRF) funded by the Ministry of Science, ICT (2017M3C4A7068177).}
\bibliographystyle{alpha}
\bibliography{bib}
\end{document}